\documentclass[10pt]{article} 

\usepackage[preprint]{rlj}           
\usepackage{bbm}
\usepackage{rl-theory}
\usepackage[ruled]{algorithm2e}

\usepackage{amssymb}            
\usepackage{amsthm}
\usepackage{mathtools}          
\usepackage{mathrsfs}           
\usepackage{graphicx}           
\usepackage{subcaption}         
\usepackage[space]{grffile}     
\usepackage{url}                
\usepackage{lipsum}             

\title{Randomized Exploration for Linear Bandits via\\ Absolute Perturbations}

\setrunningtitle{Randomized Exploration for Linear Bandits via Absolute Perturbations}

\author{Toshinori Kitamura\textsuperscript{1,$\dagger$}, Shuai Liu\textsuperscript{1,$\dagger$}, Csaba Szepesv\'{a}ri\textsuperscript{1}}

\emails{\{toshinor,shuai14,szepesva\}@ualberta.ca}

\affiliations{
$^{1}$\textbf{Department of Computing Science, University of Alberta, Canada}
\par 
$^\dagger$ Equal contribution
}

\contribution{
We introduce Absolute Thompson Sampling (ATS) a simple modification of Thompson Sampling that ensures optimism in expectation while retaining the same computational structure as TS.
}
{
While standard TS analyses rely on anti-concentration arguments and explicit handling of non-optimistic rounds \citep{agrawal2013thompson,abeille2017linear}, ATS enables a regret analysis conceptually similar to that of UCB.
}

\contribution{
To address the $\sqrt{d}$ regret gap between ATS and UCB, we propose Ensemble Absolute Thompson Sampling (EATS), which replaces a single perturbation with the maximum over multiple perturbations to more closely approximate the UCB objective.
}
{
None
}


\summary{
In stochastic linear bandits, the canonical Upper Confidence Bound (UCB) algorithm admits a simple frequentist regret analysis but can be computationally demanding, while Thompson Sampling (TS) is computationally attractive yet typically harder to analyze due to its non-optimistic nature. We propose Absolute Thompson Sampling (ATS), a simple modification of TS that ensures optimism in expectation by replacing the signed exploration noise with its absolute value. This preserves the computational efficiency of TS while avoiding the technically involved anti-concentration arguments common in TS analyses, enabling a simple UCB-style regret analysis. We show that ATS achieves $\widetilde{O}(d^{3/2}\sqrt{K})$ regret, matching existing bounds for TS in linear bandits. We further introduce Ensemble Absolute Thompson Sampling (EATS), which takes the maximum over multiple absolute perturbations with normalization by the ensemble size. As the ensemble size grows, EATS converges to the UCB objective, recovering UCB behavior in the limit. Experiments show that moderate ensemble sizes already yield strong performance. Our results point to a bridge between randomized exploration and deterministic optimism both in theory and practice.
}

\begin{document}

\maketitle  

\begin{abstract}
In stochastic linear bandits, the canonical Upper Confidence Bound (UCB) algorithm admits a simple frequentist regret analysis but can be computationally demanding, while Thompson Sampling (TS) is computationally attractive yet typically harder to analyze due to its non-optimistic nature. We propose Absolute Thompson Sampling (ATS), a simple modification of TS that ensures optimism in expectation by replacing the signed exploration noise with its absolute value. This preserves the computational efficiency of TS while avoiding the technically involved anti-concentration arguments common in TS analyses, enabling a simple UCB-style regret analysis. We show that ATS achieves $\widetilde{O}(d^{3/2}\sqrt{K})$ regret, matching existing bounds for TS in linear bandits. We further introduce Ensemble Absolute Thompson Sampling (EATS), which takes the maximum over multiple absolute perturbations with normalization by the ensemble size. As the ensemble size grows, EATS converges to the UCB objective, recovering UCB behavior in the limit. Experiments show that moderate ensemble sizes already yield strong performance. Our results point to a bridge between randomized exploration and deterministic optimism both in theory and practice.
\end{abstract}

\section{Introduction}

\looseness=-1
We study stochastic linear bandit problems, where a learner sequentially selects arms from a given arm set and receives noisy linear rewards. Efficient exploration is the central challenge for achieving low regret. Two of the most widely used approaches are optimism-based algorithms, such as Upper Confidence Bound (UCB), and randomized algorithms, such as Thompson Sampling (TS).

\looseness=-1
The UCB algorithm selects arms according to the \emph{Optimism in the Face of Uncertainty} (OFU) principle, maximizing an upper confidence bound on the estimated reward. This principle leads to a simple and elegant regret analysis based on confidence sets and the elliptical potential lemma \citep{abbasi2011improved}. However, UCB can be computationally expensive, as it requires maximizing an elliptical bonus term over the arm set.
In contrast, TS selects arms by sampling a parameter from a posterior distribution and acting greedily with respect to the sampled parameter. TS is computationally attractive since each round reduces to solving a linear maximization problem over the arm set. However, its frequentist regret analysis is more involved than that of UCB. Because the sampled parameter may underestimate the true reward, typical analyses require that TS produces an optimistic parameter with some probability and control the regret incurred during non-optimistic rounds \citep{agrawal2013thompson,abeille2017linear}.

\looseness=-1
Thus, UCB and TS have complementary computational and analytical properties. We bridge this gap by introducing a simple modification of TS that mimics the behavior of UCB. In \cref{sec:ATS}, we introduce \textbf{Absolute Thompson Sampling (ATS)}, which replaces the signed exploration noise in TS with its absolute value. This modification ensures \emph{optimism in expectation}: the exploration bonus used in ATS matches the UCB bonus in expectation. As a result, the regret analysis becomes conceptually similar to that of UCB and avoids explicitly handling non-optimistic rounds. We show that ATS achieves $\widetilde{O}(d^{3/2}\sqrt{K})$ regret with high probability, matching the existing TS regret bound for linear bandits \citep{abeille2017linear}.
Moreover, ATS retains the computational efficiency of TS, requiring only one additional linear maximization per round.

\looseness=-1
While ATS achieves the same regret bound as TS, it incurs an additional $\sqrt{d}$ factor compared to UCB \citep{abbasi2011improved}. To address this gap, we propose an extension of ATS that more closely resembles UCB. Motivated by extreme-value properties of Gaussian random variables (\cref{lemma:motivation-EATS}), we introduce \textbf{Ensemble Absolute Thompson Sampling (EATS)}, which replaces the single perturbation in ATS with the maximum over multiple perturbations. We show that, as the ensemble size grows, the EATS objective converges to the UCB objective, recovering UCB behavior in the limit. Empirically, moderate ensemble sizes already yield strong performance, suggesting a practical interpolation between randomized exploration and deterministic optimism (\cref{sec:experiments}). Establishing theoretical guarantees for EATS with a finite ensemble size remains an interesting open problem (\cref{subsec:ucb-in-limit}).

\looseness=-1
Together, our results highlight a new perspective on randomized exploration: by carefully shaping the exploration term, one can design algorithms that retain the computational efficiency of TS while exhibiting analysis and behavior closely aligned with UCB. Our work opens new directions for designing simple and efficient randomized algorithms for bandit problems.

\section{Related Work}

\looseness=-1
For linear bandits, \citet{agrawal2013thompson} and \citet{abeille2017linear} provided the fundamental frequentist regret analyses of TS, establishing $\widetilde{O}(d^{3/2}\sqrt{K})$ bounds.
A common strategy in these works is to show that TS samples an optimistic parameter with constant probability $p$, and then bound the regret by scaling the optimistic regret by $1/p$ \citep{janz2023ensemble}.
Similar proof techniques were later extended to multinomial logistic bandits \citep{oh2019thompson}, generalized linear bandits \citep{kveton2020randomized}, and linear MDPs \citep{zanette2020frequentist}. \citet{abeille2025and} developed a tighter analysis that avoids relying on constant-probability optimism, achieving $\widetilde{O}(d \sqrt{K})$ regret, matching the lower bound for linear bandits \citep{dani2008stochastic,rusmevichientong2010linearly}. However, their result requires additional structural assumptions on the action set such as the $L_2$ unit ball.

\looseness=-1
While these works establish sublinear regret guarantees, existing linear TS analyses are technically involved, as they must carefully control the regret incurred during non-optimistic rounds.
In contrast, by ensuring \emph{optimism in expectation}, our approach avoids the need to separately analyze non-optimistic rounds, leading to a proof that is conceptually as simple as that of UCB \citep{abbasi2011improved}. See \cref{sec:ATS} for more details of our approach.

\looseness=-1
Closer to our work, \citet{zhang2022feel} and \citet{huix2023tight} modify TS to enforce stronger optimism. However, their approaches rely on sampling from skewed posterior distributions via  Markov Chain Monte Carlo, which increases computational complexity and diminishes the practical appeal of TS. 
In contrast, \citet{vaswani2019old} modify UCB by introducing random perturbations. While they achieve $\widetilde{O}(d\sqrt{K})$ regret, their method requires maximizing an elliptical bonus term over the arm set, which can be computationally expensive in practice. Our method, by contrast, introduces only a minimal modification to TS and can be implemented using the same linear maximization oracle as standard TS.

\section{Preliminary}

\looseness=-1
Given two vectors $x,y \in \R^d$ and a positive definite matrix $A \in \R^{d\times d}$, we denote $\langle x,y\rangle = x^\top y$ and $\norm{x}_A = \sqrt{x^\top A x}$. We denote by $\dS^{d-1}$ the unit sphere in $\R^d$. All-zero and all-one vectors are denoted by $\bzero$ and $\bone$, respectively. For random variables $X_n$ and $X$, we write $X_n \Pto X$ to denote convergence in probability, i.e., for any $\varepsilon > 0$, $\lim_{n\to\infty}\P(\abs{X_n-X}>\varepsilon)=0$.

\subsection{Stochastic Linear Bandits}
\looseness=-1

\looseness=-1
Let $\cA \subset \R^d$ be the compact set of arms. Without loss of generality, we assume $\norm{\phi}_2 \leq 1$ for any $\phi \in \A$.
At each round $k$, the agent pulls an arm $\phi_{k} \in \A$ and observes the reward $r_k = \paren*{\theta^\star}^\top \phi_k + \varepsilon_k$.
Here, $\theta^\star\in \R^d$ is unknown to the agent and satisfies \(\norm{\theta^\star}_2 \leq B\), and $\varepsilon_k$ is zero-mean $R$-sub-Gaussian noise.
Define $r(\phi) \coloneqq (\theta^\star)^\top \phi$. The agent's goal is to achieve sublinear regret:
\begin{align}
{\textstyle
\regret (K) \df 
\sum^K_{k=1} 
r^\star - r(\phi_k)
= o\paren*{K}
}
\end{align}
where $r^\star = \max_{\phi \in \A} r(\phi)$.
We denote $\phi^\star \in \argmax_{\phi \in \A} r(\phi)$ as the optimal arm.

\looseness=-1
We denote $\Lambda_k$ as the regularized Gram matrix and $\hat{\theta}_k$ as the ridge estimator of $\theta^\star$ at round $k$:
\begin{align}
\Lambda_k = \lambda I + \sum_{i=1}^{k-1} \phi_i\phi_i^\top, \quad
\hat{\theta}_k = \Lambda_k^{-1}\sum_{i=1}^{k-1}\phi_i r_i,
\end{align}
where $\lambda > 0$ is a regularization parameter.

\subsection{Basic Algorithms}

\looseness=-1
The two popular algorithms are Upper Confidence Bound (UCB) and Thompson Sampling (TS).

\paragraph{UCB algorithm.}
UCB constructs a confidence set for $\theta^\star$ and selects the arm with the largest upper confidence bound.
The following lemma gives a confidence bound for the estimation error $\hat{\theta}_k - \theta^\star$.

\begin{lemma}[Confidence bound; \citealp{abbasi2011improved}]\label{lemma:confidence-ellipsoid}
\looseness=-1
Let $\sE$ be the event such that for all $k$,
\begin{align*}
\norm*{\hat{\theta}_{k}-\theta^\star}_{\Lambda_{k}} \leq 
\sqrt{\lambda}B + R \sqrt{d \ln \left(\frac{1+k / \lambda}{\delta}\right)} \fd \beta_k\;.
\end{align*}
where $\delta > 0$. Then, for any $\delta>0$, $\mathbb{P}(\sE) \geq 1-\delta$, and consequently, for any $\phi \in \A$ and $k$, 
\begin{align}\label{eq:reward-confidence}
 \abs{\phi^\top \theta^\star - \phi^\top \hat{\theta}_k} \leq \beta_k \norm*{\phi}_{\Lambda_k^{-1}}\;.
\end{align}
\end{lemma}

\looseness=-1
Using \cref{eq:reward-confidence}, UCB selects the arm with the largest upper confidence bound:
\begin{align}\label{eq:UCB}
    \text{(UCB)}\qquad
\phi_k^{\mathrm{UCB}} \in \argmax_{\phi\in\A}\ \phi^\top\hat{\theta}_k + \beta_k\norm{\phi}_{\Lambda_k^{-1}}.
\end{align}
This rule embodies the \emph{Optimism in the Face of Uncertainty} (OFU) principle: it favors arms with high estimated rewards and large uncertainty, thereby balancing exploitation and exploration.
UCB has a regret of $\cO(d\sqrt{K}\ln (K\delta^{-1}))$ with probability at least $1-\delta$ \citep{abbasi2011improved}.

\paragraph{Thompson sampling.}
TS maintains a posterior distribution over the unknown parameter $\theta^\star$ and acts greedily with respect to a sampled parameter.
When Gaussian priors and likelihoods are used for estimating $\theta^\star$, the posterior distribution is also Gaussian, and TS selects the arm by:
\begin{align}\label{eq:TS}
\text{(TS)}\quad
\phi_k^{\TS} \in \argmax_{\phi \in \A} \ f(\phi, \xi_k) \coloneqq
\phi^\top \hat{\theta}_k + \beta_k \phi^\top \Lambda_k^{-1/2}\xi_k
\quad\text{where} \quad
\xi_k \sim \cN(\bzero,I).
\end{align}

\looseness=-1
TS is often preferred in practice because it is computationally simpler and easier to implement. For a fixed Gaussian sample, \cref{eq:TS} reduces to a linear maximization over $\A$, whereas UCB in \cref{eq:UCB} maximizes a linear term plus an ellipsoidal bonus.

\looseness=-1
However, regret analysis for TS is more challenging than for UCB. Unlike UCB, TS may select non-optimistic arms, which requires a more delicate analysis of non-optimistic rounds. A typical proof strategy first shows that TS samples an optimistic parameter with constant probability $p > 0$ via anti-concentration arguments (e.g., Lemma 2 in \citealp{agrawal2013thompson}), and then scales the optimistic regret by $1/p$ (e.g., \emph{master theorem} in \citealp{janz2023ensemble}). This yields a regret bound of $\cO(p^{-1}\sqrt{d^3 K}\ln (K\delta^{-1}))$ \citep{agrawal2013thompson,abeille2017linear}.

\section{Absolute Thompson Sampling}\label{sec:ATS}

\looseness=-1
As explained in the previous section, UCB admits a simple regret analysis but can be computationally intensive, whereas TS is computationally efficient but requires a more involved analysis. To bridge this gap, we design a randomized algorithm that retains the computational efficiency of TS while allowing a regret analysis closer to that of UCB.

\looseness=-1
Our key idea is to make a minimal modification to TS so that it behaves more like UCB.
Specifically, we use the following equality:
For any symmetric positive definite matrix $\Lambda$ and vector $\phi \in \mathbb{R}^d$, 
\begin{align}\label{eq:TS-to-UCB}
\E_{\xi \sim \cN(\bzero, I)}\brack*{\abs*{\phi^\top \Lambda^{-1/2}\xi}} 
= \sqrt{\frac{2}{\pi}}\norm*{\phi}_{\Lambda^{-1}} \;.
\end{align}
Thus, for a fixed $\phi$, the absolute inner-product term in TS, namely $\abs{\phi^\top \Lambda_k^{-1/2} \xi_k}$, behaves like the UCB bonus term $\|\phi\|_{\Lambda_k^{-1}}$ in expectation.

\looseness=-1
Motivated by this observation, we replace the signed inner product in \cref{eq:TS} with its absolute value. 
\begin{align}\label{eq:TS-abs}
\text{(ATS)}\quad
\phi_k^{\ATS} \in \argmax_{\phi \in \A} \ \phi^\top \hat{\theta}_k + \sqrt{\frac{\pi}{2}}\beta_k \abs*{\phi^\top \Lambda_k^{-1/2}\xi_k}
\quad\text{where} \quad
\xi_k \sim \cN(\bzero,I).
\end{align}
We call this algorithm absolute Thompson sampling (ATS).

\subsection{Properties of ATS}

\looseness=-1
Let $\cF_k \coloneqq \sigma(\xi_1,\varepsilon_1,\ldots,\xi_k,\varepsilon_k)$ be the filtration such that $\varepsilon_k$ and $\xi_k$ are $\cF_k$-measurable. By slightly abusing notation, we use the shorthand $\E_{\xi_k}\brack*{\cdot} \df \E\brack*{\cdot \given \cF_{k-1}}$. 

\begin{remark}[Optimism in expectation]\label{remark:optimism-ATS}
\looseness=-1
UCB pulls the arm with the largest upper confidence bound, which provides an upper bound on the optimal reward (see \cref{eq:reward-confidence}). The ATS update rule is \emph{more optimistic than UCB in expectation}, as the following lemma shows.

\begin{lemma}\label{lemma:optimism-ATS}
Suppose that the event $\sE$ in \cref{lemma:confidence-ellipsoid} holds. 
Then, for any $k$,
\begin{align*}
    \paren*{\phi^\star}^\top\theta^\star \leq \paren*{\phi_k^{\UCB}}^\top\hat{\theta}_k + \beta_k\norm*{\paren*{\phi_k^{\UCB}}^\top}_{\Lambda_k^{-1}}
    \leq \E_{\xi_k}\brack*{\paren*{\phi^{\ATS}_k}^\top\hat{\theta}_{k} + \sqrt{\frac{\pi}{2}}\beta_k \abs*{\paren*{\phi^{\ATS}_k}^\top\Lambda_k^{-1/2} \xi_k}}\;.
\end{align*}
\end{lemma}
\begin{proof} 
It holds that
\begin{align*}
    0 &\numeq{\geq}{a} \paren*{\phi^\star}^\top\theta^\star - \left[\paren*{\phi_k^{\UCB}}^\top\hat{\theta}_k + \beta_k\norm*{\paren*{\phi_k^{\UCB}}^\top}_{\Lambda_k^{-1}}\right]\\
    &\numeq{=}{b} \paren*{\phi^\star}^\top\theta^\star - \E_{\xi_k}\brack*{\paren*{\phi^{\UCB}_k}^\top\hat{\theta}_{k} + \sqrt{\frac{\pi}{2}}\beta_k \abs*{\paren*{\phi^{\UCB}_k}^\top\Lambda_k^{-1/2} \xi_k}}\\
    &\numeq{\geq}{c}\paren*{\phi^\star}^\top\theta^\star - \E_{\xi_k}\brack*{\paren*{\phi_k^{\ATS}}^\top \hat{\theta}_k + \sqrt{\frac{\pi}{2}}\beta_k \abs*{\paren*{\phi_k^{\ATS}}^\top \Lambda_k^{-1/2}\xi_k}}\;,
\end{align*}
where (a) follows from the definition of $\phi_k^{\UCB}$ in \cref{eq:UCB} together with the reward confidence bound \cref{eq:reward-confidence} under $\sE$, (b) uses \cref{eq:TS-to-UCB}, and (c) follows from the arm-selection rule of ATS.
\end{proof}
\end{remark}

\begin{remark}[Computational efficiency]\label{remark:ATS}
\looseness=-1
The additional absolute value only minimally changes the implementation of TS. Indeed, \cref{eq:TS-abs} is equivalent to evaluating both signs of $\xi_k$ and selecting the better one. Specifically, let
\begin{align*}
\phi_k^+ \in \argmax_{\phi \in \A} f\left(\phi, \sqrt{\tfrac{\pi}{2}}\xi_k\right) \quad \text{and}\quad
\phi_k^- \in \argmax_{\phi \in \A} f\left(\phi, -\sqrt{\tfrac{\pi}{2}}\xi_k\right),
\end{align*}
where $f(\phi,\xi)$ denotes the TS objective defined in \cref{eq:TS}. Then
\begin{align}\label{eq:ATS-simple}
\phi_k^{\ATS} = 
\begin{cases}
\phi_k^+ & \text{if } f\left(\phi_k^+, \sqrt{\tfrac{\pi}{2}}\xi_k\right) \geq f\left(\phi_k^-, -\sqrt{\tfrac{\pi}{2}}\xi_k\right)\\
\phi_k^- & \text{otherwise}    
\end{cases}\;.
\end{align}
This implementation requires only one additional call to the optimization oracle compared to standard TS. 
The full procedure is summarized in \cref{alg:TS-abs}.
\end{remark}

\begin{algorithm}[t!]
\caption{Absolute Thompson Sampling (ATS)}
\label{alg:TS-abs}
\LinesNumbered
    \For{$k = 1, \dots, K$}{
        Update Gram matrix $\Lambda_k = \lambda I + \sum_{i=1}^{k-1} \phi_i \phi_i^\top$\;
        Sample $\xi_k \sim \cN(\bzero, I)$\;
        Compute $\phi_k \in \argmax_{\phi \in \A} \ \phi^\top \hat{\theta}_k + \sqrt{\frac{\pi}{2}}\beta_k \abs*{\phi^\top \Lambda_k^{-1/2}\xi_k}$. See \cref{remark:ATS}\;
        Pull $\phi_k$ and observe $r_k$\;
    }
\end{algorithm}

\subsection{Regret Analysis}

\looseness=-1
The minimal algorithmic modification by absolute value significantly simplifies the regret analysis of TS. This section proves that \cref{alg:TS-abs} achieves sublinear regret with high probability.
\begin{theorem}\label{thm:ATS}
   For any $\delta > 0$, with probability at least $1-\delta$, for any $K \geq 1$, \cref{alg:TS-abs} achieves regret of $\regret(K) = \cO\paren*{\sqrt{d^3 K}\ln (K\delta^{-1})}$.
\end{theorem}

\begin{proof}
\looseness=-1
The proof mirrors the regret analysis of UCB \citep{abbasi2011improved}: Decompose the regret with optimism (\textbf{Step 1}) and then bound the remaining terms via the elliptical potential lemma (\textbf{Step 3}). Since the optimism holds in expectation, we add one step to remove the expectations over the samples $\xi_k$ before applying the elliptical potential lemma (\textbf{Step 2}). 
Notably, unlike the standard analysis of TS (e.g., \citealp{agrawal2013thompson,abeille2017linear}), \textbf{we do not need to establish a constant probability of optimism or separately handle non-optimistic rounds}, which significantly simplifies the proof.

\paragraph{Step 1. Regret decomposition with optimism.}
\looseness=-1
Suppose that the event $\sE$ in \cref{lemma:confidence-ellipsoid} holds.
By adding and subtracting the expected objective of ATS $\E_{\xi_k}\brack*{\phi_k^\top \hat{\theta}_k + \sqrt{\frac{\pi}{2}}\beta_k \abs*{\phi_k^\top \Lambda_k^{-1/2}\xi_k}}$, we decompose the regret by:
\begin{equation}
\begin{aligned}
\regret(K)
=& \sum_{k=1}^K \underbrace{(\phi^\star)^\top\theta^\star - \E_{\xi_k}\brack*{\phi_k^\top \hat{\theta}_k + \sqrt{\frac{\pi}{2}}\beta_k \abs*{\phi_k^\top \Lambda_k^{-1/2}\xi_k}}}_{\leq 0 \text{ by \cref{lemma:optimism-ATS}}}\\
&+\underbrace{\sum_{k=1}^K \paren*{\E_{\xi_k}\brack*{\phi_k^\top \hat{\theta}_k}  - \phi_k^\top\theta^\star}}_{\fd \circled{1}} + \underbrace{\sqrt{\frac{\pi}{2}}\sum_{k=1}^K \beta_k \E_{\xi_k}\brack*{\abs*{\phi_k^\top \Lambda_k^{-1/2}\xi_k}}}_{ \eqqcolon \circled{2}} \;.\label{eq:regret-init}
\end{aligned}
\end{equation}

\looseness=-1
Therefore, it suffices to bound $\circled{1}$ and $\circled{2}$. The term $\circled{2}$ is bounded by
\begin{align}
\circled{2}
&\numeq{\leq}{a} \sqrt{\frac{\pi}{2}}\beta_K \sum_{k=1}^K\sqrt{\E_{\xi_k}\brack*{\norm{\xi_k}_2^2}} \sqrt{\E_{\xi_k}\brack*{\norm{\phi_k}_{\Lambda^{-1}_k}^2}}\label{eq:loose-bound-ATS}\\
&\numeq{=}{b} \sqrt{\frac{\pi d}{2}}\beta_K \sum_{k=1}^K\sqrt{\E_{\xi_k}\brack*{\norm{\phi_k}_{\Lambda^{-1}_k}^2}}
\numeq{\leq}{c} \sqrt{\frac{\pi d}{2}}\beta_K \sqrt{K\sum_{k=1}^K\E_{\xi_k}\brack*{\norm{\phi_k}_{\Lambda^{-1}_k}^2}}\;,\nonumber
\end{align}
where (a) uses the Cauchy--Schwarz inequality: $\E[X^\top Y] \leq \E[\norm{X}_2 \norm{Y}_2]\leq \sqrt{\mathbb{E}\|X\|_2^2} \sqrt{\mathbb{E}\|Y\|_2^2}$, (b) uses $\E_{\xi}\brack{\norm{\xi}_2^2} = d$ for $\xi \sim \cN(\bzero, I)$, and (c) uses Cauchy--Schwarz again.

\looseness=-1
By inserting the bound of $\circled{2}$ into \cref{eq:regret-init},
\begin{align}\label{eq:regret-bandit-mid}
    \regret(K)
\leq \underbrace{\sum_{k=1}^K \E\brack*{\phi_k^\top\hat{\theta}_{k} \given \cF_{k-1}} - \phi_k^\top\theta^\star}_{=\circled{1}} + \underbrace{\sqrt{\frac{\pi d}{2}}\beta_K\sqrt{K \sum_{k=1}^K \E\brack*{\norm{\phi_k}_{\Lambda^{-1}_k}^2\given \cF_{k-1}}}}_{\fd \circled{3}}\;.
\end{align}

\paragraph{Step 2. Removing expectations.}
\looseness=-1
The confidence bound \cref{eq:reward-confidence} suggests that $\circled{3}$ is on the order of $\sum^K_{k=1} \norm{\phi_k}_{\Lambda^{-1}_k}$.
This motivates us to bound $\circled{1}$ and $\circled{3}$ using the elliptical potential lemma:
\begin{lemma}[Elliptical potential; \citealp{abbasi2011improved}]\label{lemma:elliptical potential} For all $K\geq 1$ and $\lambda \geq 1$,
$$
\sum_{k=1}^K \left\|\phi_{k}\right\|_{\Lambda_{k}^{-1}}^2 \leq 2 d \ln \left(1+\frac{K}{\lambda d}\right)
\implies
\sum_{k=1}^K \left\|\phi_{k}\right\|_{\Lambda_{k}^{-1}} \leq \sqrt{2 d K \ln \left(1+\frac{K}{\lambda d}\right)}\;.
$$
\end{lemma}
However, to use \cref{lemma:elliptical potential}, we need to remove the expectations in $\circled{1}$ and $\circled{3}$ over the samples $\xi_k$. We use the following lemma to convert expectations to concrete realizations.

\begin{lemma}\label{lemma:exp-to-concrete-events}
\looseness=-1
Let $\sE_1$ and $\sE_2$ be the events such that, for any $K \geq 1$,
\begin{align}
    \sE_1: \qquad & 
    \sum_{k=1}^K \E\brack*{\phi_k^\top\hat{\theta}_{k} \given \cF_{k-1}} \leq \sum_{k=1}^K \phi_k^\top\hat{\theta}_{k} + 2\paren*{B + \frac{\beta_K}{\sqrt{\lambda}}} \sqrt{K \ln \frac{\pi^2 K^2}{6\delta}}\;,\label{eq:exp-to-concrete-1}\\
    \sE_2: \qquad & 
    \sum_{k=1}^K \E\brack*{\norm{\phi_k}_{\Lambda^{-1}_k}^2\given \cF_{k-1}}
    \leq 2\sum_{k=1}^K \norm{\phi_k}_{\Lambda^{-1}_k}^2 + \frac{4}{\lambda} \ln \frac{2K}{\delta}\;,\label{eq:exp-to-concrete-2}
\end{align}
where $\delta > 0$. Then, for any $\delta>0$, $\P\paren*{\sE \cap \sE_1^c} \leq \delta$ and $\P\paren*{\sE_2^c} \leq \delta$.
\end{lemma}
\begin{proof}
\looseness=-1
For the first claim, under the good event $\sE$, we have
$$
\abs*{\phi_k^\top\hat{\theta}_{k}}\leq 
\norm{\hat{\theta}_k}_2 \leq \norm{\theta^\star}_2 + \norm{\hat{\theta}_k - \theta^\star}_2 \leq 
\norm{\theta^\star}_2 + \frac{1}{\sqrt{\lambda}}\norm{\hat{\theta}_k - \theta^\star}_{\Lambda_k} \leq B + \frac{\beta_K}{\sqrt{\lambda}}
\;.
$$
Then, the first claim is followed by applying a version of Azuma--Hoeffding (\cref{lemma:hoeffding-and-good}) to the martingale difference sequence $\E\brack*{\phi_k^\top\hat{\theta}_{k} \given \cF_{k-1}} - \phi_k^\top\hat{\theta}_{k}$.
The second claim uses Lemma D.4 in \citet{rosenberg2020near} (see \cref{lemma:exp-to-concrete}) with the fact that $\norm{\phi_k}_{\Lambda^{-1}_k}^2 \leq \frac{1}{\lambda}$.
\end{proof}

\paragraph{Step 3. Applying elliptical potential lemma.}
\looseness=-1
Now suppose that the event $\sE\cap \sE_1\cap\sE_2$ holds. This happens with probability at least $1-3\delta$.\footnote{
    This is because $\P\paren*{\sE\cap \sE_1 \cap \sE_2} = \P\paren*{\sE \cap \sE_1} - \P\paren*{\sE \cap \sE_1 \cap \sE_2^c} \geq 1-2\delta - \P\paren*{\sE_2^c} \geq 1-3\delta$.
} Then the first term $\circled{1}$ in \cref{eq:regret-bandit-mid} is bounded by
\begin{align*}
\circled{1}
\numeq{\leq}{a} \sum_{k=1}^K \phi_k^\top\hat{\theta}_{k} - \phi_k^\top\theta^\star + 2\paren*{B + \frac{\beta_K}{\sqrt{\lambda}}} \sqrt{K \ln \frac{\pi^2 K^2}{6\delta}}
\numeq{\leq}{b} \beta_K \sum_{k=1}^K \norm*{\phi_k}_{\Lambda^{-1}_k} + 2\paren*{B + \frac{\beta_K}{\sqrt{\lambda}}} \sqrt{K \ln \frac{\pi^2 K^2}{6\delta}}\;,
\end{align*}
where (a) uses $\sE_1$ and (b) uses \cref{eq:reward-confidence} under $\sE$.
Using $\sE_2$, we have 
\begin{align*}
\circled{3}
\numeq{\leq}{a} \beta_K\sqrt{\frac{\pi d K}{2}}\sqrt{2\sum_{k=1}^K \norm{\phi_k}_{\Lambda^{-1}_k}^2 + \frac{4}{\lambda} \ln \frac{2K}{\delta}}\;.
\end{align*}

\looseness=-1
Finally, by inserting the bounds for $\circled{1}$ and $\circled{3}$ into \cref{eq:regret-bandit-mid},
\begin{align*}
\regret(K)
\leq & \beta_K {\color{red}\sum_{k=1}^K \norm*{\phi_k}_{\Lambda^{-1}_k}} + 2\paren*{B + \frac{\beta_K}{\sqrt{\lambda}}} \sqrt{K \ln \frac{\pi^2 K^2}{6\delta}}
+ \beta_K\sqrt{\frac{\pi d K}{2}}\sqrt{{\color{red}2\sum_{k=1}^K \norm{\phi_k}_{\Lambda^{-1}_k}^2} + \frac{4}{\lambda} \ln \frac{2K}{\delta}} \\
\numeq{\leq}{a} & \beta_K {\color{red}\sqrt{2 d K \ln \left(1+\frac{K}{\lambda d}\right)}} + 2\paren*{B + \frac{\beta_K}{\sqrt{\lambda}}} \sqrt{K \ln \frac{\pi^2 K^2}{6\delta}} + \beta_K\sqrt{\frac{\pi d K}{2}}\sqrt{{\color{red}4 d \ln \left(1+\frac{K}{\lambda d}\right)} + \frac{4}{\lambda} \ln \frac{2K}{\delta}}\\
\numeq{=}{b} & \cO\paren*{\sqrt{d^3K}\ln (K\delta^{-1})}\;,
\end{align*}
where (a) applies \cref{lemma:elliptical potential} to the red terms and (b) uses $\beta_K = \cO(\sqrt{d\ln K \delta^{-1}})$.
\end{proof}

\section{Ensemble Absolute Thompson Sampling}\label{sec:ensemble}

\looseness=-1
Although ATS is computationally efficient, its regret bound incurs an additional $\sqrt{d}$ factor compared to UCB. This looseness stems from the fact that the key identity \eqref{eq:TS-to-UCB} holds only for a fixed $\phi$ independent of the sample $\xi$, whereas in ATS the selected action $\phi_k^{\ATS}$ depends on the sample $\xi_k$. Consequently, the ATS objective can be overly optimistic in expectation relative to UCB, as suggested by \cref{lemma:optimism-ATS}.

\looseness=-1
In this section, we explore a simple extension of ATS that can mitigate this $\sqrt{d}$ gap by more closely approximating the UCB objective. \textbf{Our analysis provides only asymptotic justification, and a non-asymptotic regret guarantee remains open.} We discuss the details in \cref{subsec:ucb-in-limit}.

\looseness=-1
The key idea is to leverage extreme-value properties of Gaussian random variables.
\begin{proposition}\label{lemma:motivation-EATS}
Let $\xi_1, \dots, \xi_N \sim \mathcal{N}(\bzero,I)$ be i.i.d. Gaussian vectors. For fixed symmetric positive definite matrix $\Lambda \in \R^{d \times d}$ and vector $\phi \in \mathbb{R}^d$, we have
\begin{align}\label{eq:ensemble-to-UCB}
    \max_{i \in [N]} \frac{1}{\sqrt{2\ln N}} \abs*{\phi^\top \Lambda^{-1/2} \xi_i} \Pto \norm*{\phi}_{\Lambda^{-1}}\;.
\end{align}
\end{proposition}

\looseness=-1
The proof is provided in \cref{sec:proof-of-EATS-motivation}.
Thus, for large $N$ and fixed $\phi$, the left hand side of \cref{eq:ensemble-to-UCB} provides an estimator of the UCB bonus $\|\phi\|_{\Lambda^{-1}}$.
This observation motivates the following algorithm which we call ensemble absolute Thompson sampling (EATS):
\begin{align}
\text{(EATS)}\quad
\phi_k^{\mathrm{EATS}} \in \arg\max_{\phi \in \mathcal{A}}\max_{i\leq N} \ \phi^\top \hat{\theta}_k + \frac{\beta_k}{\sqrt{ 2\ln N }} \lvert {\phi^\top \Lambda_k^{-1/2}\xi_k^i} \rvert 
\quad\text{where} \quad
\xi^i_k \sim \mathcal{N}(0,I).
\end{align}

\subsection{Properties of EATS}\label{section:props_eats}

\begin{remark}[Computation of EATS] \label{remark:EATS}
\looseness=-1
Similar to ATS, EATS can be implemented using the same linear maximization oracle as standard TS. EATS requires $2N$ calls to the maximization oracle per round. For each $i \in [N]$, let
\begin{align*}
\phi_k^{+, i} \in \argmax_{\phi \in \A} f(\phi, \tfrac{1}{\sqrt{2\ln N}}\xi^i_k) \quad \text{and}\quad
\phi_k^{-, i} \in \argmax_{\phi \in \A} f(\phi, -\tfrac{1}{\sqrt{2\ln N}}\xi^i_k),
\end{align*}
Then, EATS can be implemented by picking the arm $\phi_k^{+, i}$ or $\phi_k^{-, i}$ that achieve the maximum value of $f(\phi, \pm \frac{1}{\sqrt{2\ln N}}\xi^i_k)$ across all $i \in [N]$. \cref{alg:EATS} summarizes the full procedure of EATS.
\end{remark}

\begin{remark}[UCB in the limit]\label{subsec:ucb-in-limit}
\looseness=-1
From \cref{lemma:motivation-EATS} and the compactness of $\cA$, it is trivial that the EATS objective closely approximates the UCB objective when $N$ is sufficiently large. Formally, we have the following result.

\begin{proposition}\label{lemma:EATS-is-UCB} For any $k$, when $N\to \infty$, it holds that
\begin{align}\label{eq:EATS-is-UCB}
    \max_{\phi \in \cA}\max_{i \in [N]}  
    \phi^\top \hat{\theta}_k + \frac{\beta_k}{\sqrt{ 2\ln N }} \lvert {\phi^\top \Lambda_k^{-1/2}\xi_k^i} \rvert
    \Pto \max_{\phi \in \cA} \phi^\top \hat{\theta}_k + \beta_k \norm*{\phi}_{\Lambda_k^{-1}}\;.
\end{align}
\end{proposition}

\looseness=-1
For completeness, we provide the proof in \cref{sec:proof-of-EATS-is-UCB}.
Since the objective of EATS matches that of UCB, EATS recovers UCB behavior in the limit of large $N$. As a direct consequence, when $N$ is sufficiently large, EATS achieves the same regret bound as UCB, $\cO\paren*{d\sqrt{K} \ln (K\delta^{-1})}$, which is $\sqrt{d}$ improvement over the ATS bound (\cref{thm:ATS}).

\looseness=-1
However, the above argument is purely asymptotic and does not provide a non-asymptotic regret bound for EATS. If approximating UCB requires $N$ to be exponentially large in $d$, the computational cost of EATS may become prohibitive. \textbf{For what finite size $N$ EATS can achieve a UCB-matching regret bound remains an interesting open question.} Although we do not currently have a theoretical answer, our preliminary experiments in \cref{sec:experiments} suggest that moderate ensemble sizes already yield strong empirical performance, which is encouraging for the practical utility of EATS.
\end{remark}

\begin{algorithm}[t!]
\caption{Ensemble Absolute Thompson Sampling (EATS)}
\label{alg:EATS}
\LinesNumbered
    \For{$k = 1, \dots, K$}{
        Update Gram matrix $\Lambda_k = \lambda I + \sum_{i=1}^{k-1} \phi_i \phi_i^\top$\;
        Sample $\xi_k^1, \dots, \xi_k^N \sim \cN(\bzero, I)$\;
        Compute $\phi_k \in \argmax_{\phi \in \A} \max_{i\in [N]} \phi^\top \hat{\theta}_k + \frac{\beta_k}{\sqrt{2\ln N}} \abs*{\phi^\top \Lambda_k^{-1/2}\xi_k^i}$. See \cref{remark:EATS}\;
        Pull $\phi_k$ and observe $r_k$\;
    }
\end{algorithm}

\section{Experiments}\label{sec:experiments}

\looseness=-1
This section presents the numerical results for the proposed algorithms. We first compare the empirical performance of ATS and TS, with the goal of showing that ATS achieves performance comparable to that of TS. We then study the empirical behavior of EATS to illustrate the potential of a new algorithmic design that may approach the statistical performance of UCB while substantially reducing computational complexity, provided that the ensemble size is chosen reasonably.

\subsection{Comparison of $\ATS$ vs $\TS$}\label{section:comparison_ts_ats}

\looseness=-1
We conduct experiments in three environments: the unit ball, the hypercube, and the two-lane shortest-path problem. Unless otherwise specified, all reported performance metrics are averaged over 100 independent random seeds.

Across these three environments, in addition to UCB (\cref{eq:UCB}), TS (\cref{eq:TS}), and ATS (\cref{eq:TS-abs}), we also compare the following variants:
\begin{align*}
&\text{(TS-No-Inflation)}\quad
\phi_k^{\text{TS-No-Inflation}} \in \argmax_{\phi \in \A} \ \phi^\top \hat{\theta}_k + \phi^\top \Lambda_k^{-1/2}\xi_k
\quad\text{where}\quad
\xi_k \sim \cN(\bzero,I)\;.\\
&\text{(ATS-No-$\sqrt{\pi/2}$)}\quad
\phi_k^{\text{ATS-No-$\sqrt{\pi/2}$}} \in \argmax_{\phi \in \A} \ \phi^\top \hat{\theta}_k + \beta_k \abs*{\phi^\top \Lambda_k^{-1/2}\xi_k}
\quad\text{where}\quad
\xi_k \sim \cN(\bzero,I)\;.\\
&\text{(ATS-No-Inflation)}\quad
\phi_k^{\text{ATS-No-Inflation}} \in \argmax_{\phi \in \A} \ \phi^\top \hat{\theta}_k + \abs*{\phi^\top \Lambda_k^{-1/2}\xi_k}
\quad\text{where}\quad
\xi_k \sim \cN(\bzero,I)\;.
\end{align*}
In addition, because UCB can be computationally expensive, we use Rarely-Switching-UCB (RS-UCB) \citep{abbasi2011improved} to reduce the computational cost.
\paragraph{Unit Ball}
The unit ball environment is a natural benchmark in this setting, as information-theoretic lower bounds for linear bandits have been constructed with the arm set taken to be the unit ball \citep{lattimore2020bandit}. In addition, recent work has shown that linear Thompson Sampling does not require forced exploration when the action set is smooth and convex \citep{abeille2025and}.

From \cref{fig:unit-ball-sweep}, the left panel shows that ATS performs worse than TS by only a constant factor, while exhibiting the same dependence on the dimension $d$. After removing the $\sqrt{\pi/2}$ factor, as in ATS-No-$\sqrt{\pi/2}$, the resulting performance is comparable to that of TS, and is even slightly better in our experiments.

In the right panel, we further evaluate the no-inflation variants of both TS and ATS, following the algorithmic design in \cite{abeille2025and}, where TS is shown to achieve a $\tilde{\mathcal O}(d\sqrt{T})$ regret guarantee on smooth and convex action sets such as the unit ball. Our experiments support this theoretical prediction and further suggest that ATS may also be able to match the performance of TS in the same setting considered by \cite{abeille2025and}.
\begin{figure}[t] \centering 
\begin{minipage}{0.48\linewidth} \centering \includegraphics[width=\linewidth]{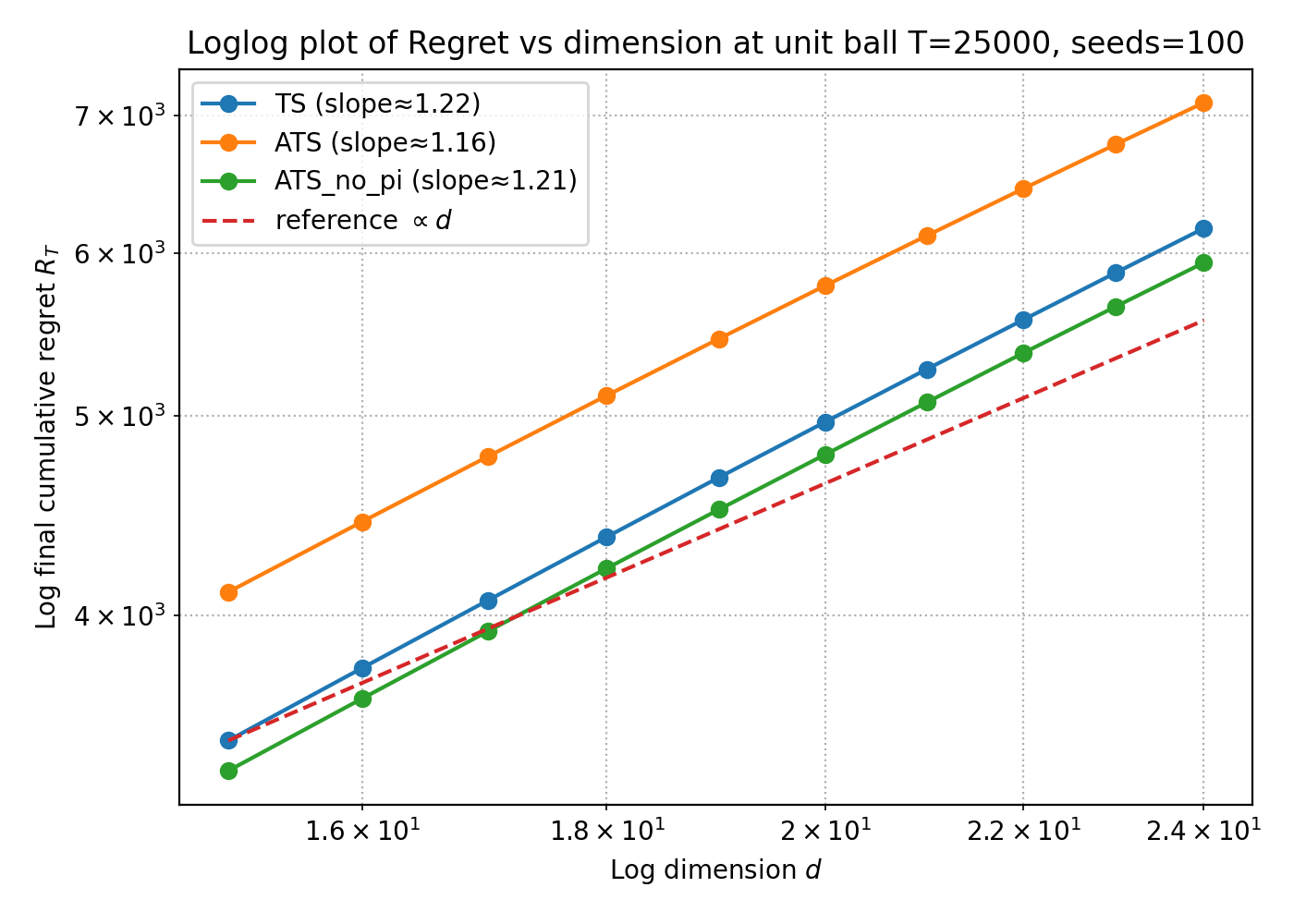} 
\end{minipage} 
\hfill 
\begin{minipage}{0.48\linewidth} \centering \includegraphics[width=\linewidth]{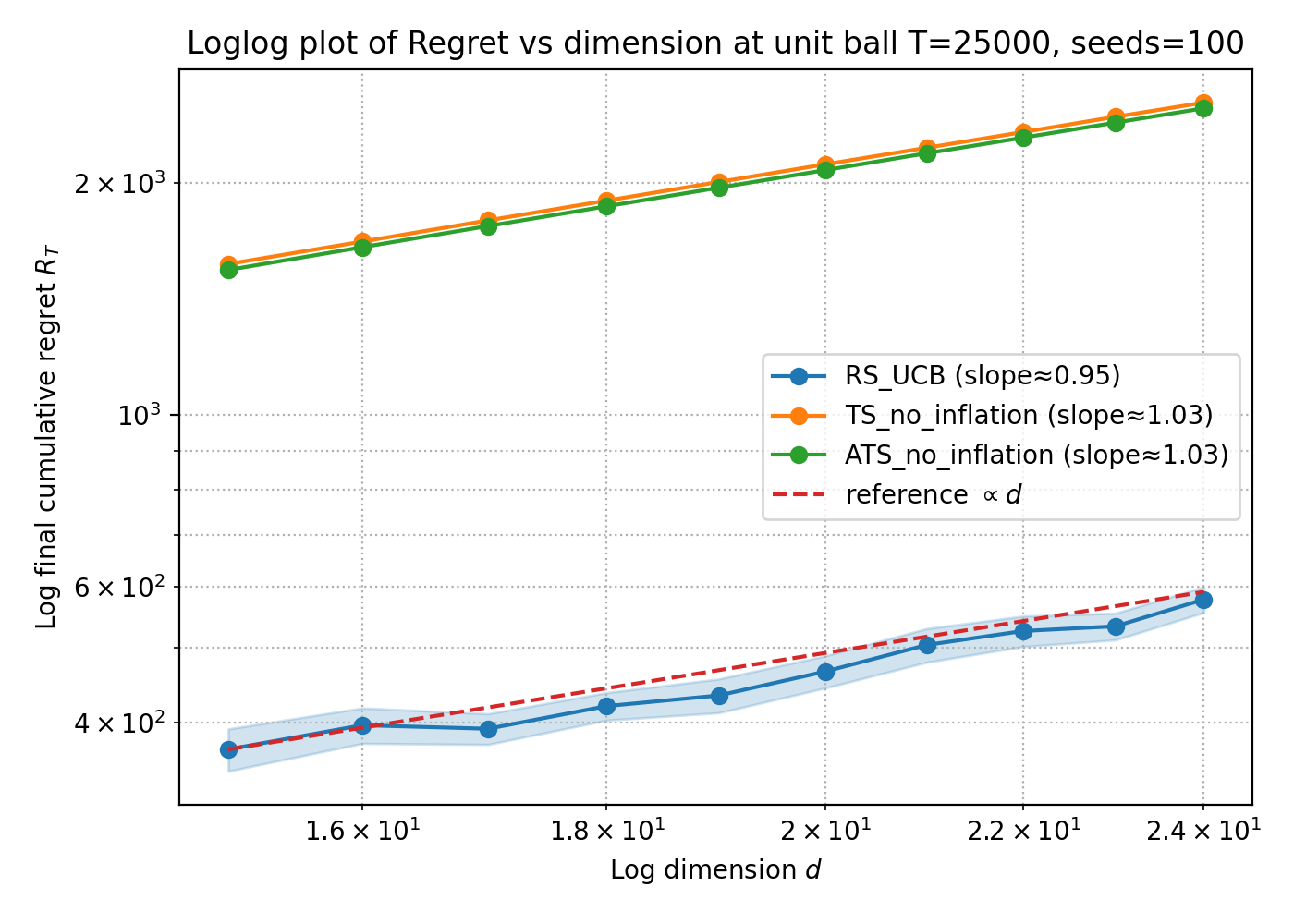} 
\end{minipage} 
\caption{\small Sweep results for the linear bandit problem with unit-sphere arm set, $\|\theta^\star\|=1$, and Gaussian noise $\eta_t \sim \mathcal N(0,0.1)$. The parameter $\theta^\star$ is sampled uniformly from the unit sphere. Results are averaged over 100 independent random seeds. Both panels present log--log plots of the final regret of different algorithms, where the feature dimension ranges over ${15,16,\dotsc,24}$ and the horizon is fixed at $T=250000$: \textit{(Left)} TS, ATS, and ATS-No-$\sqrt{\pi/2}$; \textit{(Right)} RS-UCB, TS-No-Inflation, and ATS-No-Inflation.}
\label{fig:unit-ball-sweep} 
\end{figure} 
\paragraph{Hypercube}

The hypercube environment is defined as follows. For arm dimension $d \in \mathbb{Z}^+$, the action set is
$\mathcal{A} = \{\pm 1\}^d / \sqrt{d}$,
so that every arm $\phi \in \mathcal{A}$ satisfies $\|\phi\|_2 \le 1$. The unknown parameter $\theta_*$ is sampled uniformly from the unit $\ell_2$-sphere. We consider this environment for two reasons. First, information-theoretic lower bounds for linear bandits are known to arise from hypercube-type constructions \citep{lattimore2020bandit}. Second, the action set has cardinality $2^d$, making exact UCB optimization computationally expensive in this setting.

The hypercube is also a challenging experimental regime because the number of actions grows exponentially with $d$, and the asymptotic behavior appears only slowly. From \cref{fig:hypercube-sweep}, we observe that ATS performs worse than TS and ATS-No-$\sqrt{\pi/2}$ by an approximately constant factor. As in the unit ball setting, removing the $\sqrt{\pi/2}$ factor improves the empirical performance of ATS, making it slightly better than TS in our experiments. Moreover, the right panel suggests that these results should still be interpreted as transient behavior, since none of the observed regrets appears to have reached its asymptotic scaling by the end of the horizon.
\begin{figure}[t] \centering 
\begin{minipage}{0.48\linewidth} \centering \includegraphics[width=\linewidth]{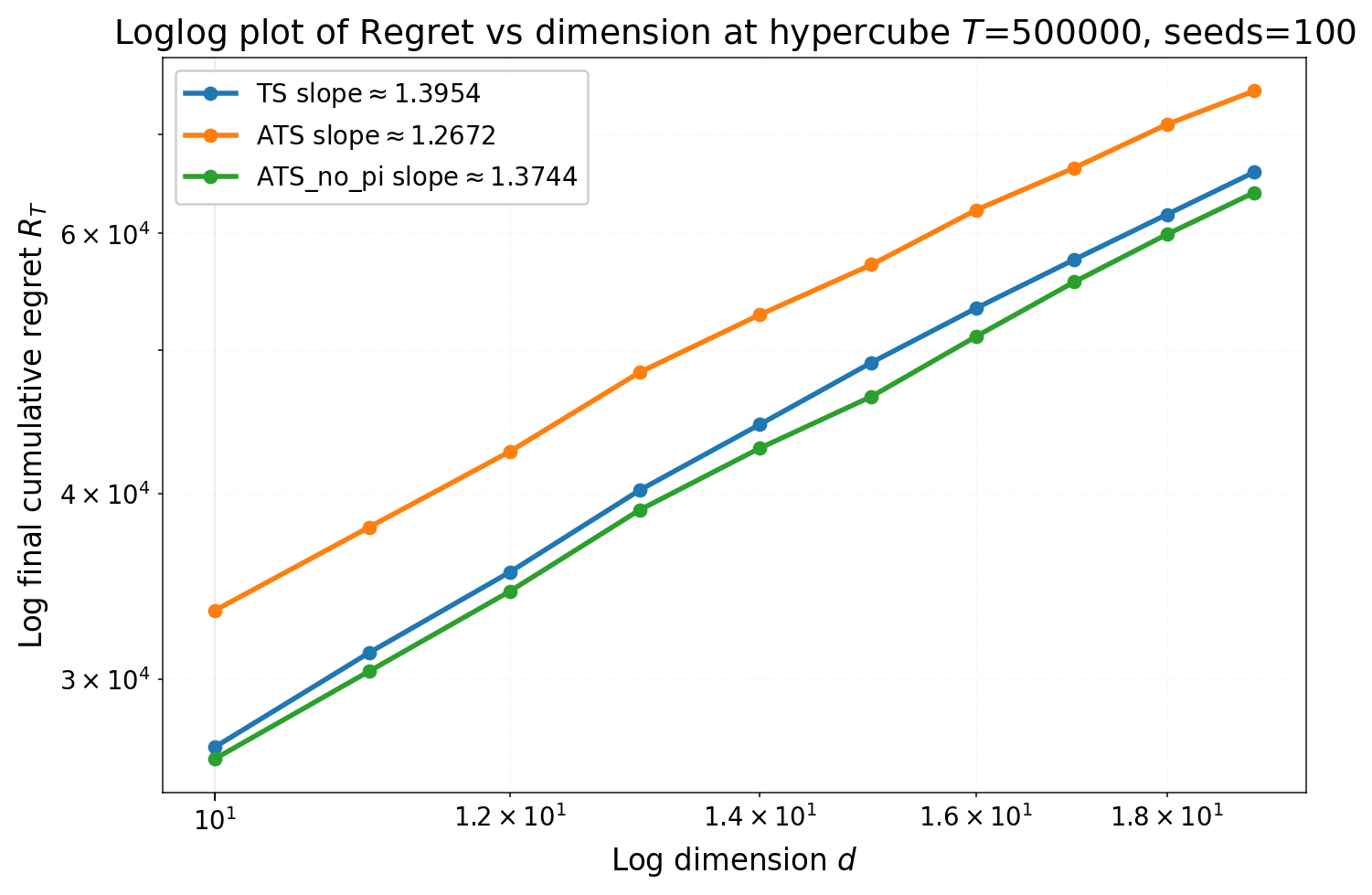} 
\end{minipage} 
\hfill 
\begin{minipage}{0.48\linewidth} \centering \includegraphics[width=\linewidth]{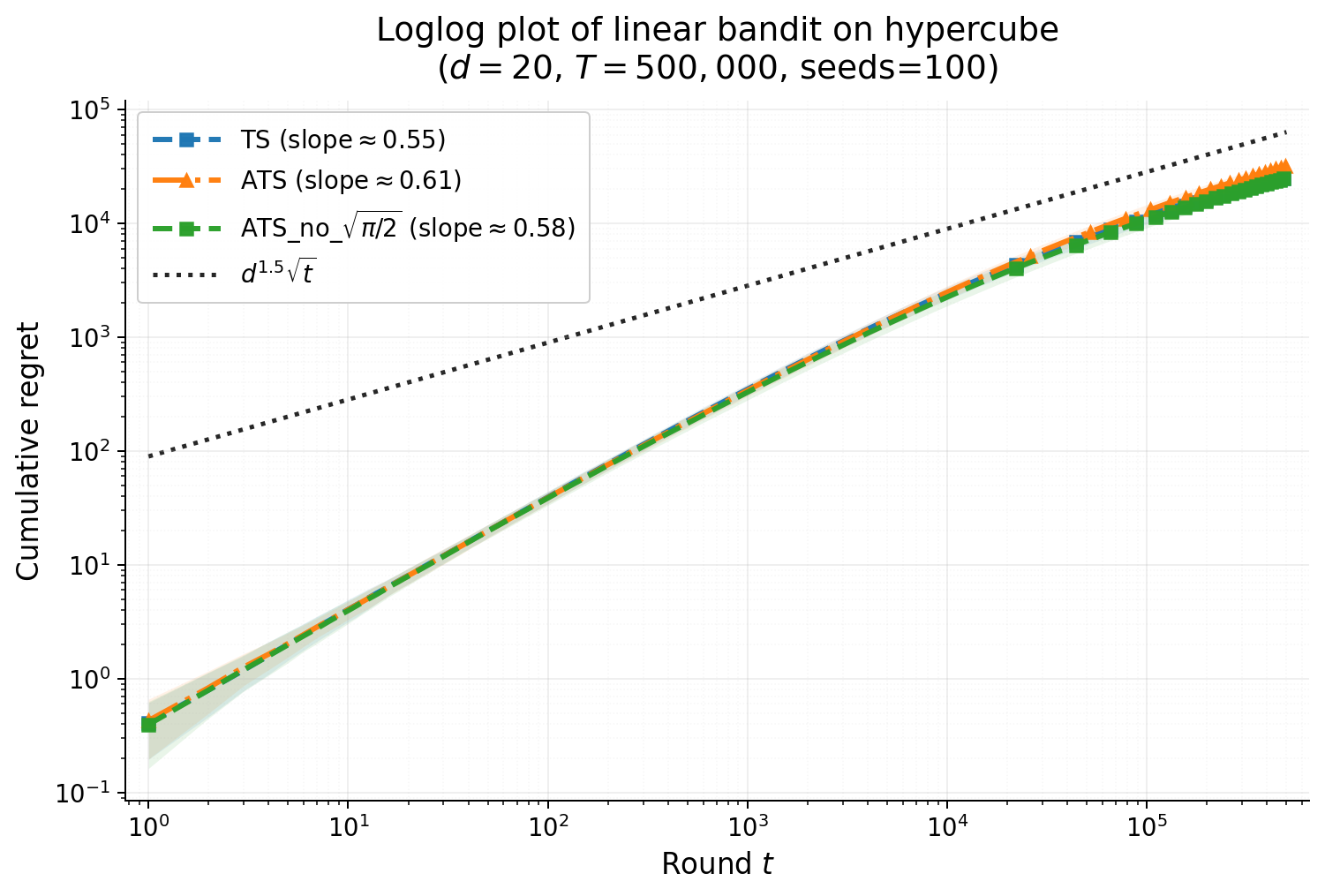} 
\end{minipage} 
\caption{\small Sweep results for the linear bandit problem with the arm set is the hypercube, $\|\theta^\star\|=1$, and noise $\eta_t \sim \mathcal N(0,0.1)$. The parameter $\theta^\star$ is drawn uniformly at random from the unit sphere. Results are averaged over 100 independent random seeds. Both panels show plots of the final regret for different algorithms, with feature dimension varying over ${10,11,\dotsc,19}$ and horizon fixed at $T=500000$: \textit{(Left)} Log--log plot of final regret of TS, ATS, and ATS-No-$\sqrt{\pi/2}$; \textit{(Right)} TS, ATS, and ATS-No-$\sqrt{\pi/2}$.} 
\label{fig:hypercube-sweep} 
\end{figure} 
There are at least two possible explanations for why the observed $d$-dependence of TS, ATS, and ATS-No-$\sqrt{\pi/2}$ does not fully match the theoretical prediction. The first is that the dimensions considered may still be too small. In \cref{fig:hypercube-sweep}, for $d \in \{20,\dotsc,30\}$ and $T=200{,}000$, all three algorithms still appear to be in the transient regime. Pushing to substantially larger $d$ is computationally infeasible in our setting, since even $d=30$ already corresponds to roughly $2^{30}$ arms. The second possibility is that the time horizon is still not large enough. Indeed, the right panel of \cref{fig:hypercube-sweep} indicates that all three methods remain in the transient regime even at the largest horizon we consider. However, horizons such as $T=500{,}000$ are already substantial relative to the dimensions studied and are also near the limit of what we can simulate under our computational budget.
\begin{figure}[t] \centering 
\begin{minipage}{0.48\linewidth} \centering \includegraphics[width=\linewidth]{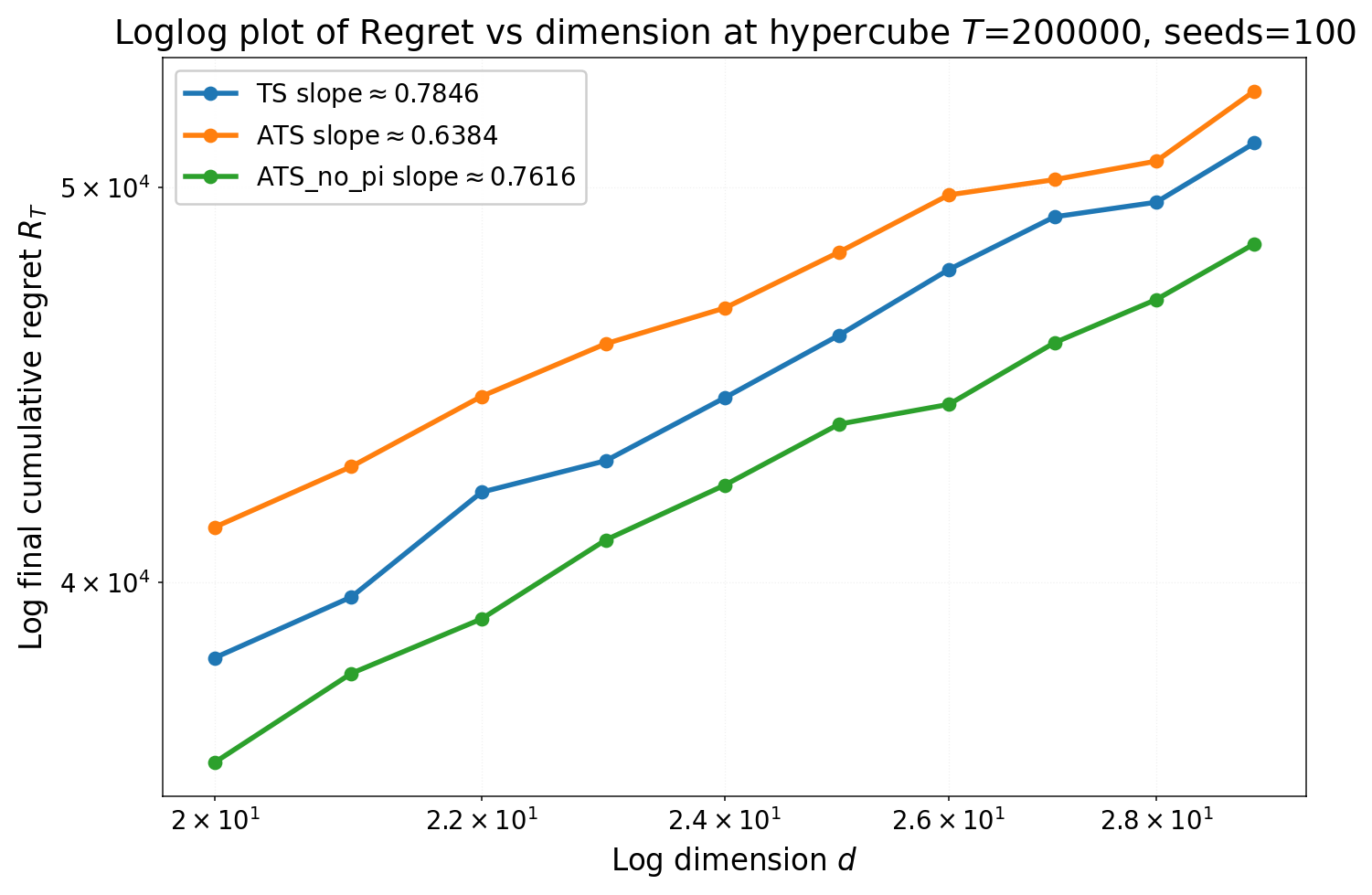} 
\end{minipage} 
\hfill 
\begin{minipage}{0.48\linewidth} \centering \includegraphics[width=\linewidth]{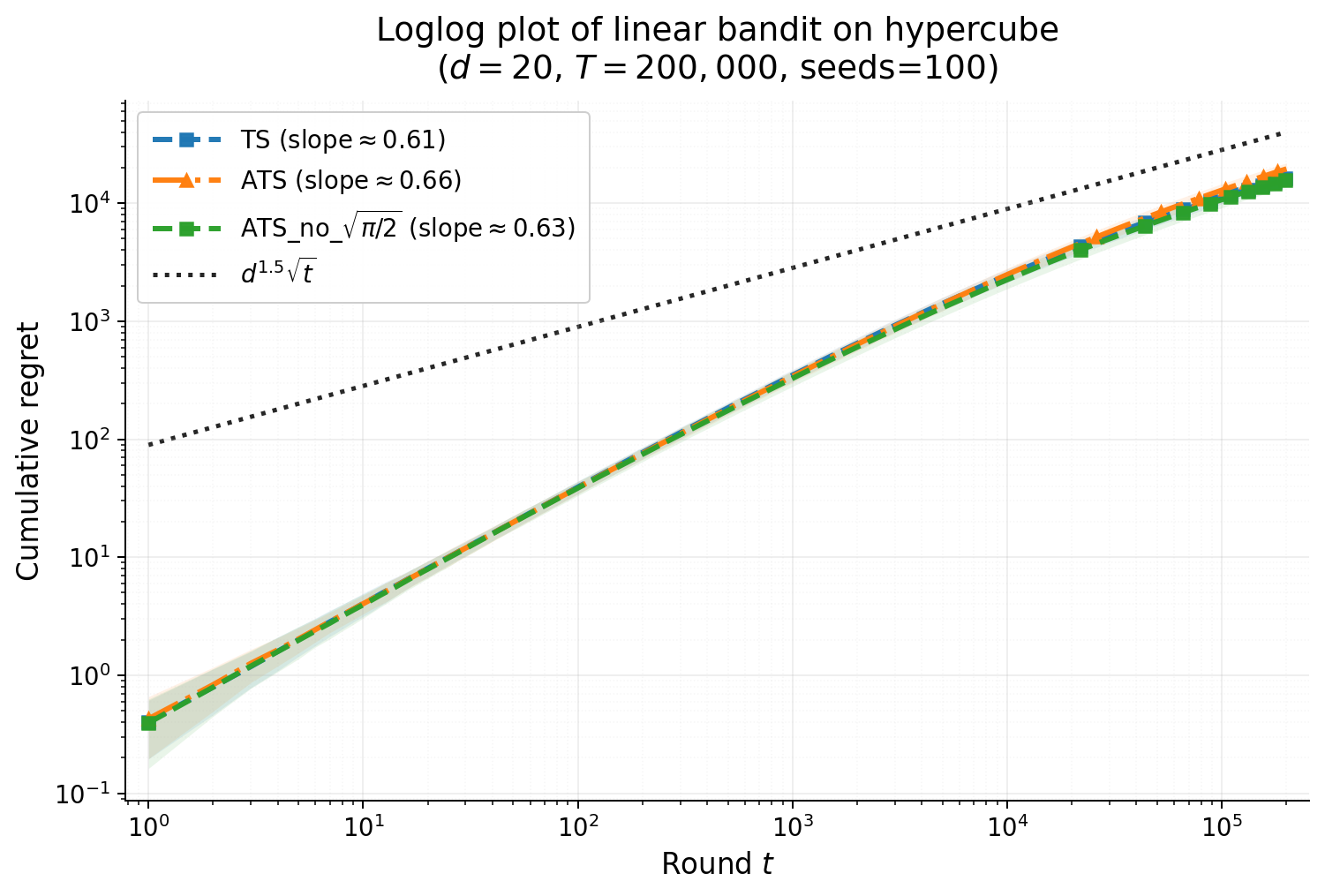} 
\end{minipage} 
\caption{\small Numerical results for the linear bandit problem with the arm set is the hypercube, $\|\theta^\star\|=1$, and noise $\eta_t \sim \mathcal N(0,0.1)$. The parameter $\theta^\star$ is drawn uniformly at random from the unit sphere. Results are averaged over 100 independent random seeds. Both panels show plots of the final regret for different algorithms, with feature dimension varying over ${20,21,\dotsc,29}$ and horizon fixed at $T=200000$: \textit{(Left)} Log--log plot of final regret of TS, ATS, and ATS-No-$\sqrt{\pi/2}$; \textit{(Right)} TS, ATS, and ATS-No-$\sqrt{\pi/2}$.} 
\label{fig:hypercube-sweep-higher} 
\end{figure}

Nevertheless, the current experiments consistently support the conclusion that ATS is worse than TS and ATS-No-$\sqrt{\pi/2}$ by an approximately constant factor, that removing the $\sqrt{\pi/2}$ term improves ATS and can even make it slightly outperform TS empirically, and that these methods appear to exhibit broadly similar dependence on $d$. It is therefore reasonable to conjecture that this qualitative picture persists in larger-$d$ and larger-$T$ regimes.

\paragraph{Two-Lane Shortest Path}

In this section, we study a shortest-path bandit instance on a two-lane layered graph. The learner's goal is to identify a minimum-cost path from a source node $s$ to a sink node $t$. The graph has $m\in\mathbb Z^+$ stages, each containing two nodes representing the upper and lower lanes. From each node at stage $i$, the learner may either stay in the same lane or switch to the other lane at stage $i+1$. Hence, each inter-stage transition admits four possible edges: $\{\text{upper, lower}\}\to\{\text{upper, lower}\}$.

We assign cost $1$ to every switching transition, cost $1$ to upper $\to$ upper, and cost $0$ to lower $\to$ lower. Consequently, the optimal path is obtained by staying in the lower lane at every stage, yielding total cost $0$. The resulting graph contains $2^m$ distinct $s$-$t$ paths and $d=2+4(m-1)+2=4m$ edges. See \cref{fig:shortest-path} as an example.

For any path $x$, let $\phi(x)\in\{0,1\}^d$ denote its edge-incidence feature vector. The observed cost is given by
$\tilde C(x)=\phi(x)^\top\theta_*+\varepsilon$, where $\varepsilon\sim\mathcal N(0,R^2)$. The edge-parameter vector $\theta_*$ assigns value $0$ to lower$\to$lower edges and value $1$ to all other edges, that is,
$(\theta_*)_i = 1-\mathbb I\!\left(e_i=\mathrm{lower}\to\mathrm{lower}\right)$.
From \cref{fig:shortest-path-sweep}, we see that ATS is worse than TS and ATS-No-$\sqrt{\pi/2}$ only by a constant factor, while all three appear to share the same scaling with the dimension $d$. Consistent with the unit ball experiments, removing the $\sqrt{\pi/2}$ factor yields a constant-factor improvement in empirical performance, and ATS-No-$\sqrt{\pi/2}$ becomes comparable to TS.

Moreover, TS-No-Inflation and ATS-No-Inflation continue to outperform the other algorithms by a significant margin. However, their empirical dependence on $d$ appears to be close to $d^{1.5}$, in contrast to the behavior observed in the unit ball setting.

\begin{figure}[t] \centering 
\begin{minipage}{0.48\linewidth} \centering \includegraphics[width=\linewidth]{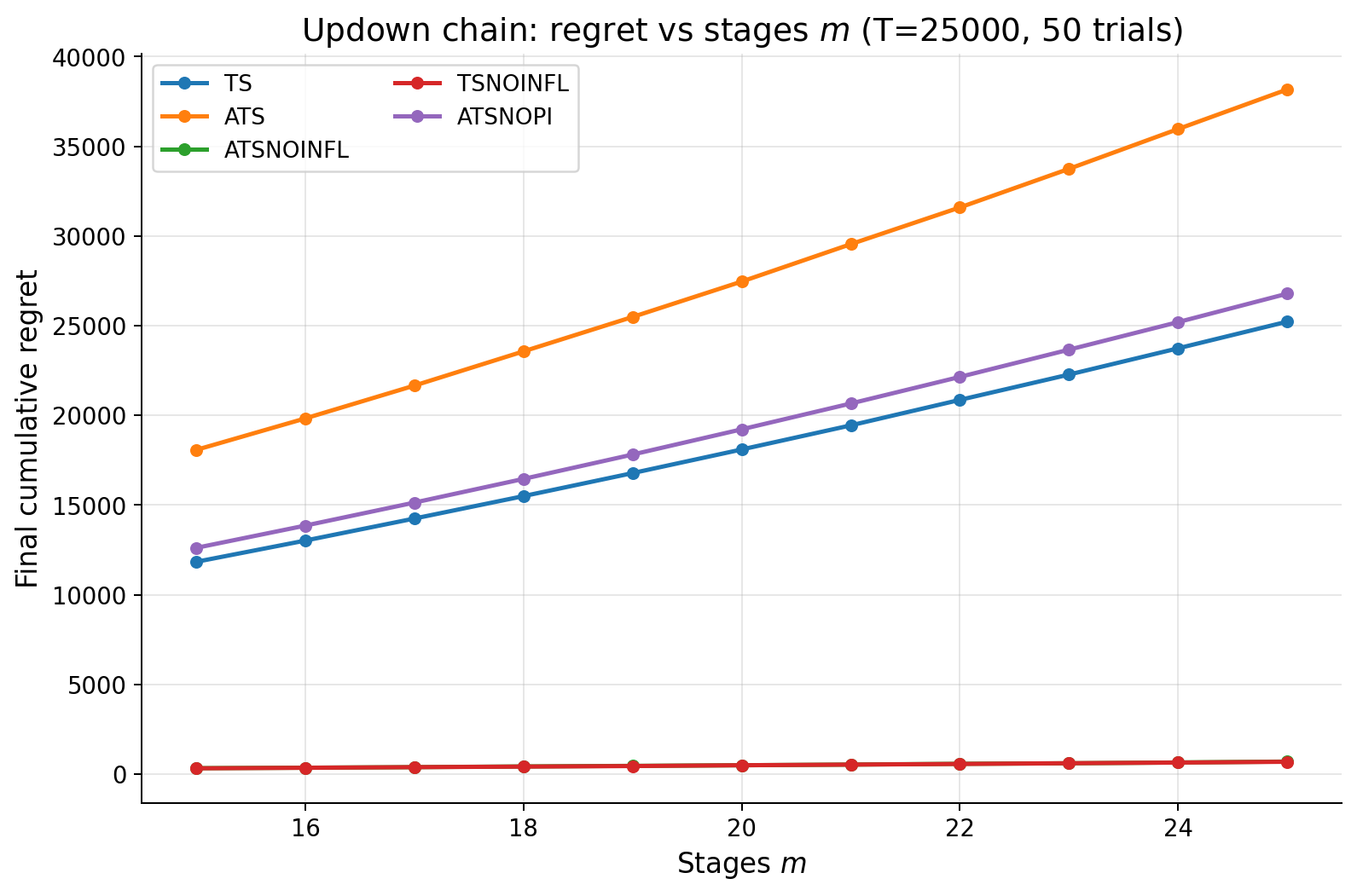} 
\end{minipage} 
\hfill 
\begin{minipage}{0.48\linewidth} \centering \includegraphics[width=\linewidth]{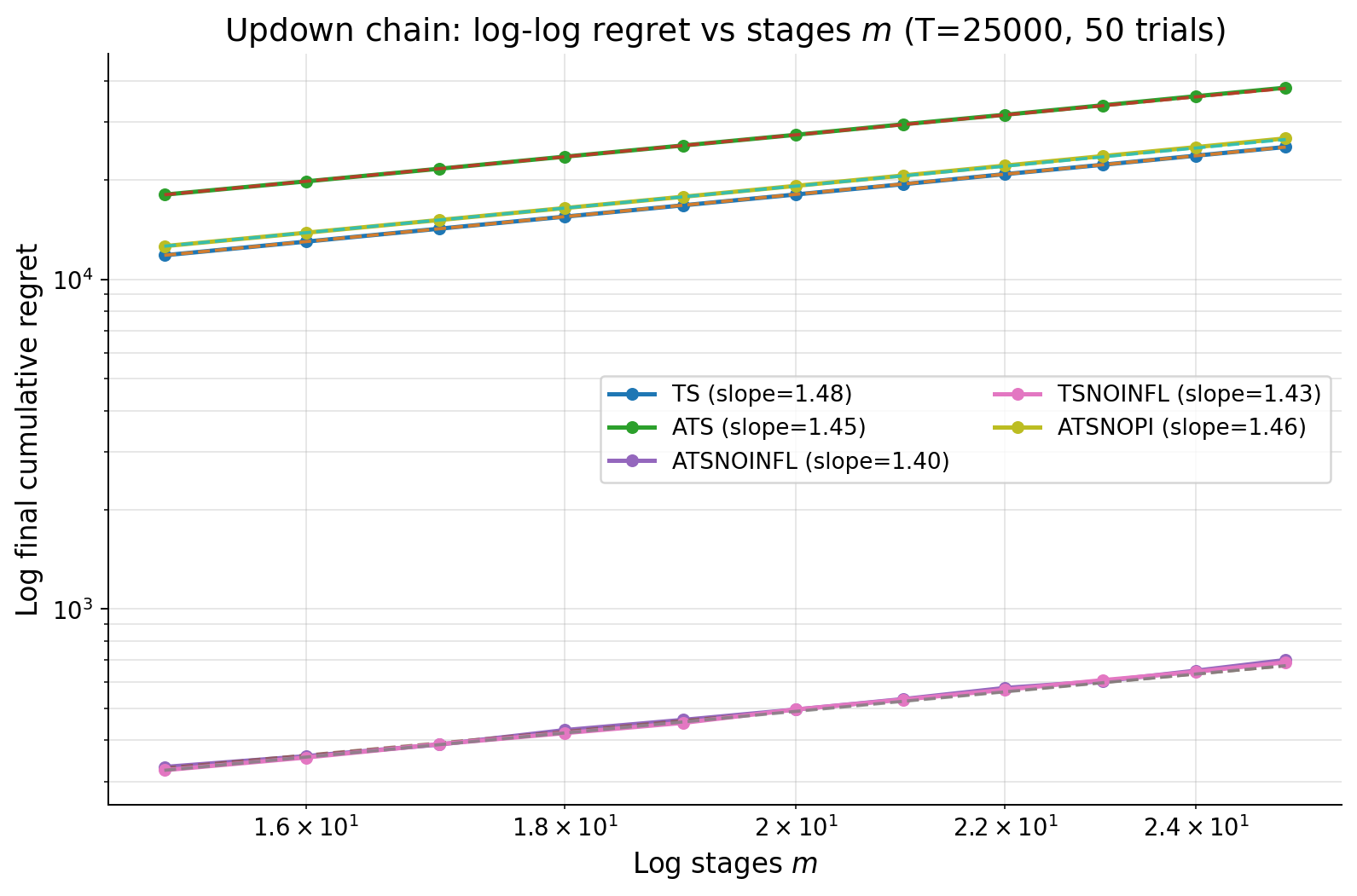} 
\end{minipage} 
\caption{\small Sweep results for the two-lane shortest-path problem with Gaussian noise $\varepsilon_t \sim \mathcal{N}(0,0.2)$, averaged over 50 independent random seeds. Both panels report the final regret of TS, TS-No-Inflation, ATS, ATS-No-$\sqrt{\pi/2}$, and ATS-No-Inflation as the feature dimension ranges over $\{15,16,\dotsc,25\}$, with horizon fixed at $T=250000$: \textit{(Left)} final regret versus $d$; \textit{(Right)} log-log plot of final regret versus $d$.}
\label{fig:shortest-path-sweep} 
\end{figure} 

\subsection{Empirical evidences for EATS}
In this section, we present empirical evidence that EATS has the potential to satisfy a $\tilde O(d\sqrt{K})$ regret guarantee when $N=\mathrm{poly}(d)$. To this end, we study the dependence of cumulative regret on the dimension $d$ through a dimension sweep. \color{red}In all EATS experiments reported in this section, we choose $N=d$. The settings of the environments in this section are also the same as the ones in \cref{section:comparison_ts_ats}.\color{black}

\paragraph{Unit Ball}
From \cref{fig:EATS-unitball}, we see that the regret of EATS exhibits an approximately linear dependence on $d$, while remaining worse than UCB by a substantial constant factor. At the same time, its empirical performance is comparable to that of TS-No-Inflation, for which \citet{abeille2025and} established a regret bound with linear dependence on $d$.
\begin{figure}[t]
    \centering
    \includegraphics[width=0.5\linewidth]{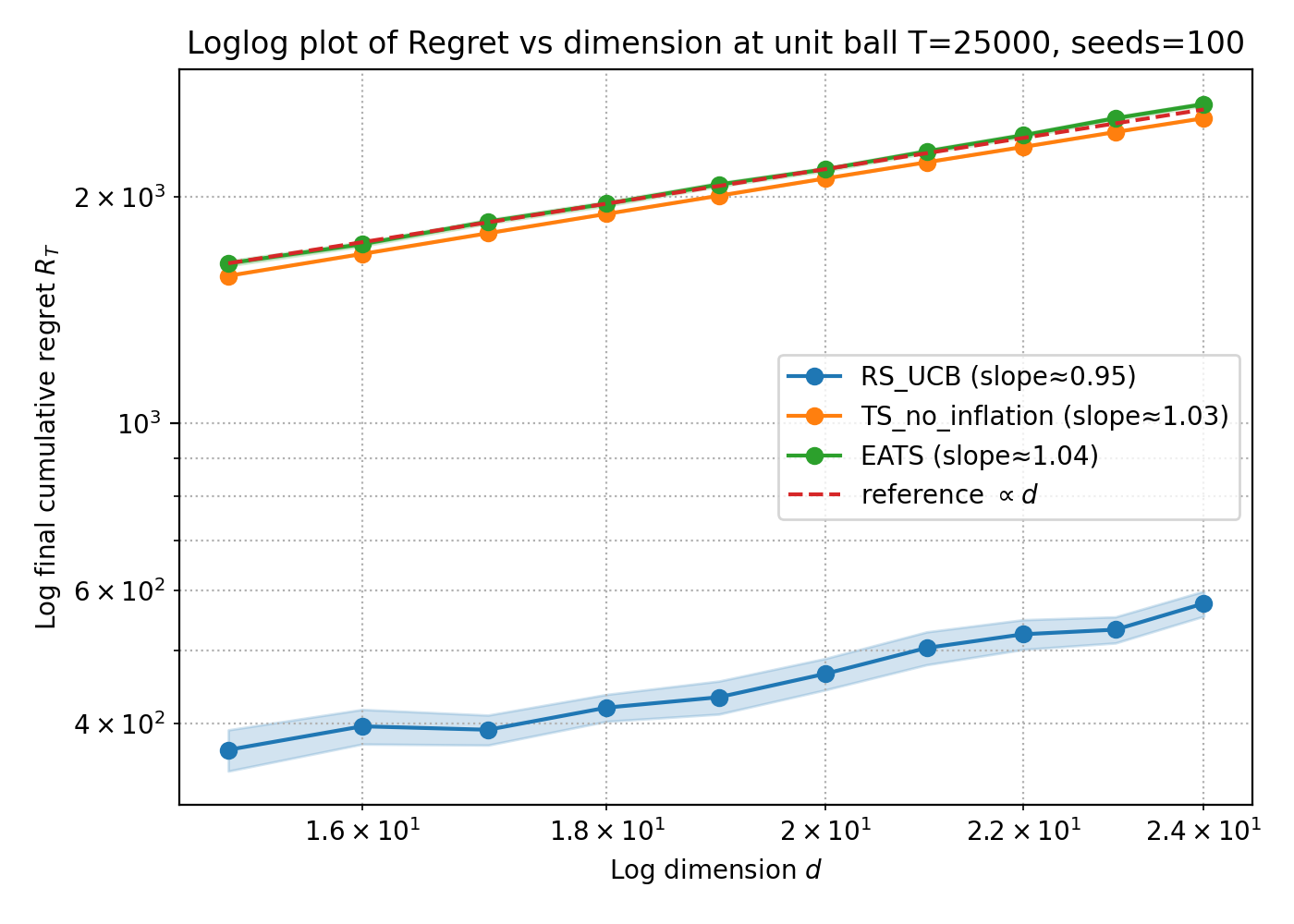}
    \caption{\small Sweep results for the linear bandit problem with unit-sphere arm set. Same setting as \cref{fig:unit-ball-sweep}. Results are averaged over 100 independent random seeds. Log--log plot of regret vs dimension for EATS, RS-UCB, and TS-No-Inflation.}
    \label{fig:EATS-unitball}
\end{figure}

\paragraph{Hypercube}
From the left panel of \cref{fig:EATS-hypercube}, we observe that the regret of EATS scales linearly with $d$. The right panel further suggests that EATS has already reached its asymptotic regime by $d=30$ and $T=200{,}000$, which in turn indicates that the experiments for $d \in \{20,\dotsc,30\}$ are likely also operating in the asymptotic regime. In contrast, \cref{fig:hypercube-sweep-higher} shows that TS and ATS still remain in their transient regimes over the same range, which explains why their observed dependence on $d$ is smaller.

Overall, the fact that EATS consistently displays linear scaling across both environments, including one in which the asymptotic regime appears to be visible, suggests that a $\tilde O(d\sqrt{K})$-type guarantee for EATS is plausible.
\begin{figure}[t] \centering 
\begin{minipage}{0.48\linewidth} \centering \includegraphics[width=\linewidth]{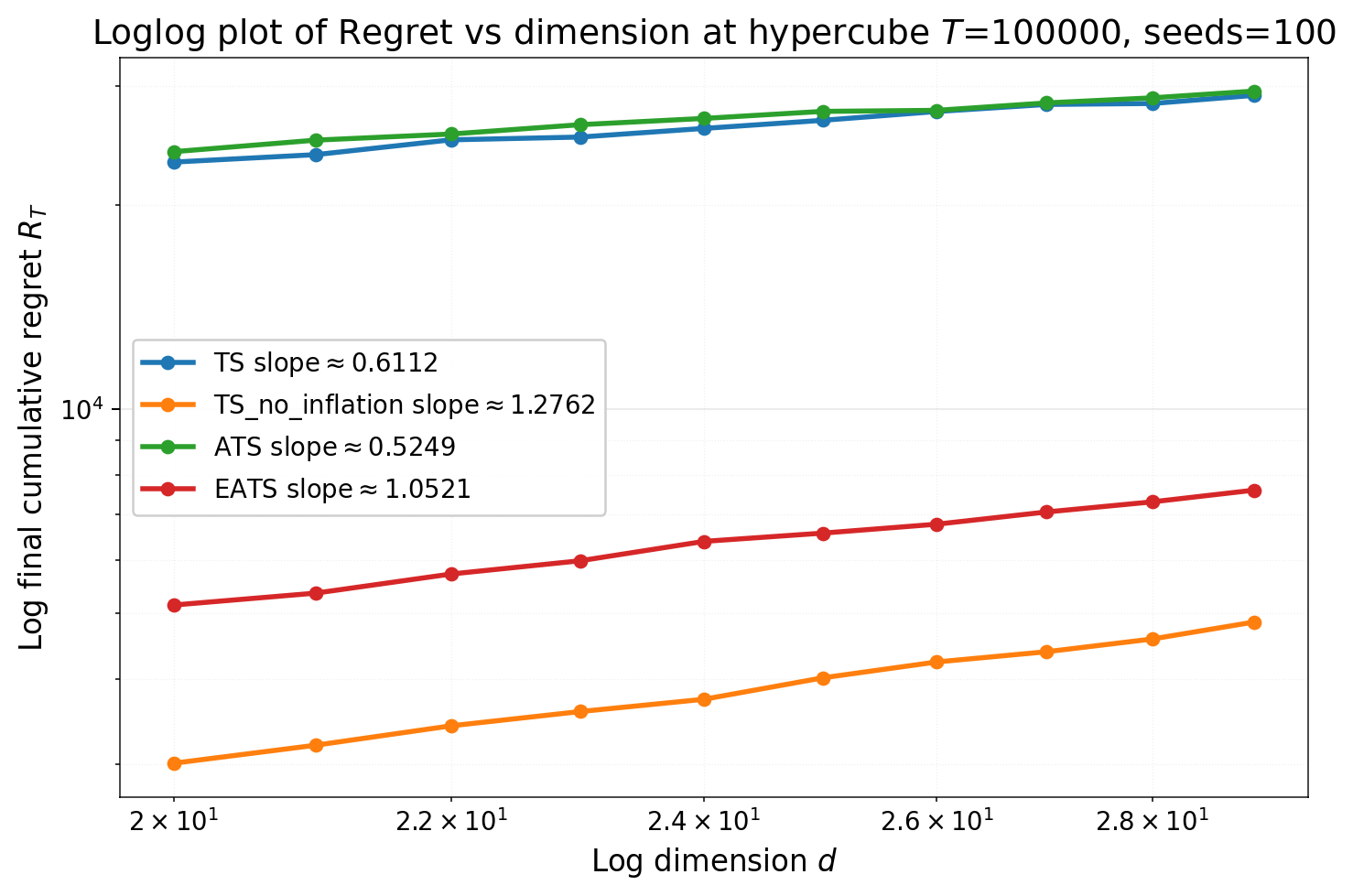} 
\end{minipage} 
\hfill 
\begin{minipage}{0.48\linewidth} \centering \includegraphics[width=\linewidth]{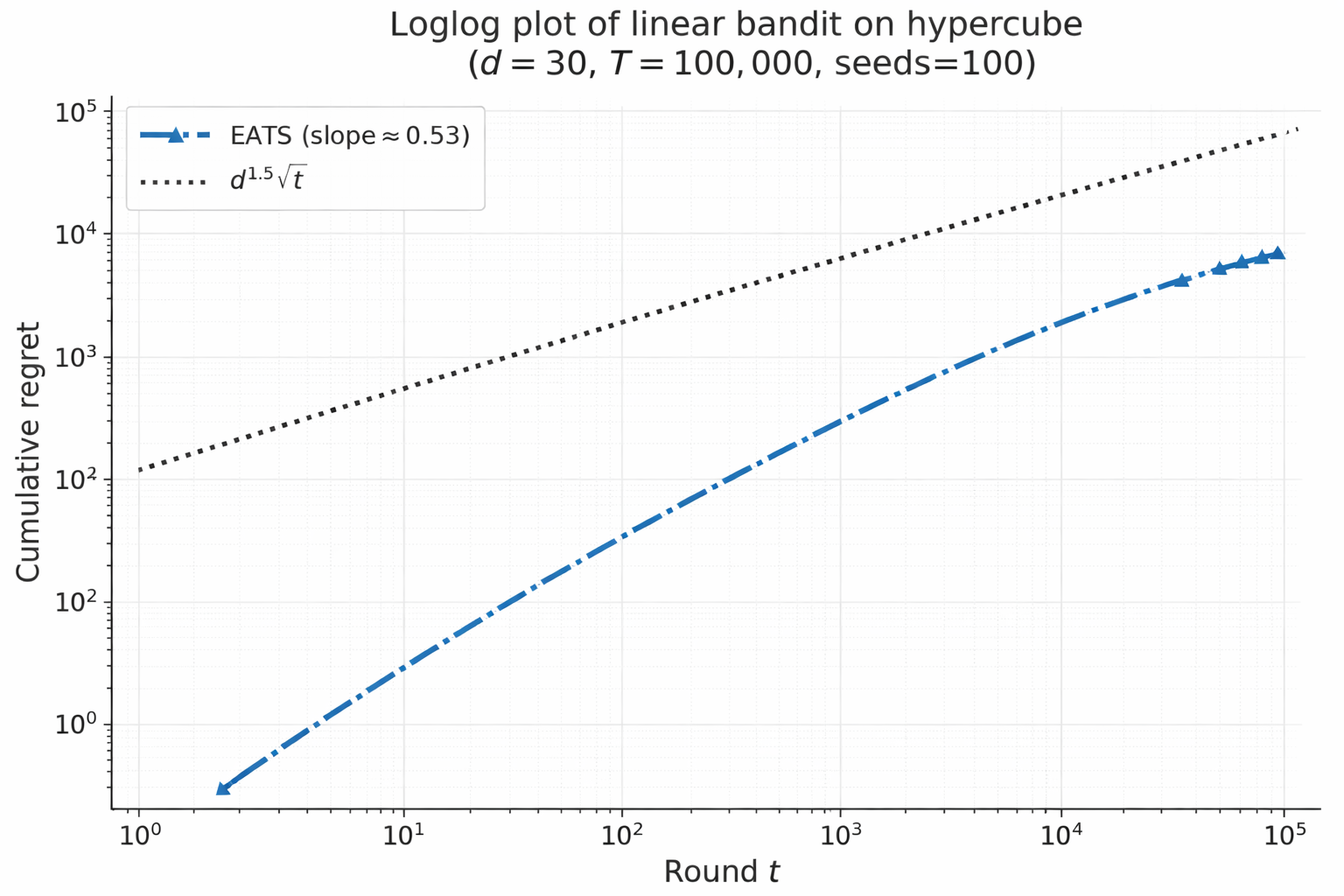} 
\end{minipage} 
\caption{\small Sweep results for the linear bandit problem with the arm set is the hypercube. Same setting as \cref{fig:hypercube-sweep}. \textit{(Left)} Log--log plot of final regret of EATS, TS, ATS, and TS-No-Inflation; \textit{(Right)} The behavior of EATS at $d=30$ and $T=100,000$.}
\label{fig:EATS-hypercube} 
\end{figure} 

\section{Conclusion}

\looseness=-1
We introduced Absolute Thompson Sampling (ATS), which replaces the signed exploration noise in TS with its absolute value. This ensures optimism in expectation, enabling a UCB-style regret analysis without the technically involved anti-concentration arguments typically required for TS. ATS achieves $\widetilde{O}(d^{3/2}\sqrt{K})$ regret while retaining the computational efficiency of randomized exploration.

\looseness=-1
We also proposed Ensemble Absolute Thompson Sampling (EATS), which approximates the UCB objective by taking the maximum over multiple absolute perturbations with normalization by the ensemble size. As the ensemble size grows, EATS recovers UCB behavior in the limit. Experiments suggest that moderate ensemble sizes can yield strong performance. Establishing non-asymptotic guarantees for EATS with finite ensemble sizes remains an interesting direction for future work.

\looseness=-1
Overall, our results suggest a new perspective on randomized exploration: by carefully shaping the exploration term, one can design algorithms that retain the computational efficiency of TS while exhibiting analysis and behavior closely aligned with UCB.




\bibliography{main}
\bibliographystyle{rlj}

\appendix
\beginSupplementaryMaterials

\section{Useful Lemmas} \label{sec:exp-to-realization-proof}

\begin{lemma}[Azuma-Hoeffding]\label{lemma:hoeffding}
    Let $(X_k, \cF_k)_{k=1}^K$ be a martingale difference sequence, i.e., $X_k$ is $\cF_k$-measurable and $\E\brack{X_k \given \cF_{k-1}} = 0$.
    Assume $X_k \in [l_k, u_k]$ almost surely. Then, for any $\delta \in (0, 1)$,
	\begin{align*}
	\P \paren*{
      \sum_{k=1}^K X_k
      \geq
      \sqrt{
        \sum_{k=1}^K \frac{(u_k - l_k)^2}{2} \log \frac{1}{\delta}
      }
    }
    \leq \delta\;.
	\end{align*}
\end{lemma}

\begin{lemma}[Azuma-Hoeffding with good events]\label{lemma:hoeffding-and-good}
    Let $(Y_k)_{k=1}^\infty$ be $\cF_k$-measurable and integrable random variables.
    Let $\sE_k$ be $\cF_{k-1}$-measurable events and let $\sE$ be an event such that $\sE \subseteq \bigcap^\infty_{k=1} \sE_k$.
    Assume that on $\sE_k$, we have $Y_k \in [l_k, u_k]$ almost surely. Then, for any $\delta \in (0, 1)$,
	\begin{align*}
	\P \paren*{
    \sE\cap
    \brace*{
        \exists K \geq 1:
      \sum_{k=1}^K \paren*{Y_k - \E[Y_k \given \cF_{k-1}]}
      \geq
      \sqrt{
        \sum_{k=1}^K \frac{(u_k - l_k)^2}{2} \log \frac{\pi^2 K^2}{6\delta}
      }
    }
    }
    \leq \delta\;.
	\end{align*}
\end{lemma}
\begin{proof}
Define $\tilde{Y}_k\coloneqq \done[\sE_k] Y_k$ which returns $Y_k$ if $\sE_k$ holds and $0$ otherwise. Since $\sE_k \in \cF_{k-1}$, the process $\tilde{X}_k \coloneqq \tilde{Y}_k - \E[\tilde{Y}_k \mid \cF_{k-1}]$ is $\cF_{k}$-measurable and a martingale difference sequence. Moreover, $\tilde{X}_k$ lies in an interval of length at most $u_k - l_k$ almost surely. Define $\delta_K \coloneqq 6\delta / \pi^2 K^2$. By \cref{lemma:hoeffding}, and taking a union bound over $K \geq 1$,
$$
\P \paren*{
\exists K \geq 1: 
    \sum_{k=1}^K \tilde{X}_k
    \geq
    \sqrt{
    \sum_{k=1}^K \frac{(u_k - l_k)^2}{2} \log \frac{1}{\delta_K}
    }
} \leq \sum_{K=1}^\infty \delta_K = \delta\;.
$$
On the event $\sE$, we have $\tilde{Y}_k = Y_k$ for all $k$. Therefore,
$$
	\P \paren*{
    \sE\cap
    \brace*{
        \exists K \geq 1:
      \sum_{k=1}^K \paren*{Y_k - \E[Y_k \given \cF_{k-1}]}
      \geq
      \sqrt{
        \sum_{k=1}^K \frac{(u_k - l_k)^2}{2} \log \frac{1}{\delta_K}
      }
    }
    }
    \leq \delta\;.
$$
\end{proof}

\begin{lemma}[\textbf{Lemma D.4} in \citet{rosenberg2020near}]\label{lemma:exp-to-concrete}
Let $\left(X_{k}\right)_{k=1}^{\infty}$ be random variables with expectation adapted to the filtration $\left(\mathcal{F}_{k}\right)_{k=0}^{\infty}$. 
Suppose that $0 \leq X_{k} \leq C$ almost surely. 
Then, with probability at least $1-\delta$, the following holds for all $K \geq 1$:
$$
\sum_{k=1}^K \mathbb{E}\left[X_{k} \mid \mathcal{F}_{k-1}\right] \leq 4 C \ln \frac{2 K}{\delta} + 2 \sum_{k=1}^K X_k\;.
$$
\end{lemma}

\section{Proofs for \cref{sec:ensemble}}

\subsection{Proof of \cref{lemma:motivation-EATS}}\label{sec:proof-of-EATS-motivation}

\begin{lemma}[Restatement of \cref{lemma:motivation-EATS}]
Let $\xi_1, \dots, \xi_N \sim \mathcal{N}(\bzero,I)$ be i.i.d. Gaussian vectors. For fixed symmetric positive definite matrix $\Lambda \in \R^{d \times d}$ and vector $\phi \in \mathbb{R}^d$, we have
\begin{align*}
    \max_{i \in [N]} \frac{1}{\sqrt{2\ln N}} \abs*{\phi^\top \Lambda^{-1/2} \xi_i} \Pto \norm*{\phi}_{\Lambda^{-1}}\;.
\end{align*}
\end{lemma}
\begin{proof}
Let $Z_1, \dots, Z_N\sim \cN(0, 1)$ be i.i.d. random variables. Then,
$$\max_{i \in [N]} \frac{1}{\sqrt{2\ln N}} \abs*{\phi^\top \Lambda^{-1/2} \xi_i} = \|\phi\|_{\Lambda^{-1}} \frac{1}{\sqrt{2\ln N}} \max_{i \in [N]} \abs{Z_i}\;.$$
Thus, it suffices to show that $\frac{1}{\sqrt{2\ln N}} \max_{i \in [N]} \abs{Z_i} \Pto 1$.

\looseness=-1
We use the shorthand $u_N \df \sqrt{2\ln N}$. 
We prove the claim by showing that, for any $\varepsilon > 0$,
\begin{equation}\label{eq:max-gauss-converge-in-prob}
\mathbb{P}\left(\left|\frac{\max _{i \in [N]}\left|Z_i\right|}{u_N}-1\right|>\varepsilon\right) \leq \mathbb{P}\left(\max_{i \in [N]} \left|Z_i\right|>(1+\varepsilon) u_N\right)+\mathbb{P}\left(\max_{i \in [N]} \left|Z_i\right| \leq(1-\varepsilon) u_N\right)
\underset{N \rightarrow \infty}{\longrightarrow} 0.
\end{equation}
We show that both the upper tail and the lower tail vanish as $N \to \infty$.

\looseness=-1
\paragraph{Upper tail bound.}
Fix $\varepsilon > 0$. Let $Z \sim \cN(0, 1)$.
Using $\mathbb{P}(|Z|>t) \leq 2 \exp\paren{-t^2 / 2}$ for $t \geq 0$,
\begin{align*}
\mathbb{P}\left( \max_{i \in [N]} \left| Z_{i} \right| > (1+\varepsilon) u_{N}\right) 
&\leq N\mathbb{P}\left( \left| Z \right| > (1+\varepsilon) u_{N}\right) 
\leq 2 N \exp\paren*{-\frac{(1+\varepsilon)^2 u_N^2}{2}}\\
&=2 N \exp\paren*{-(1+\varepsilon)^2 \ln N}=2 N^{1-(1+\varepsilon)^2}=2 N^{-\left(2 \varepsilon+\varepsilon^2\right)}\underset{N \rightarrow \infty}{\longrightarrow} 0.
\end{align*}

\looseness=-1
\paragraph{Lower tail bound.}
For the lower tail, we have
\begin{equation}\label{eq:Max gaussian lower tail}
\mathbb{P}\left(\max _{i \leq N}\left|Z_i\right| \leq t\right)=(1-\mathbb{P}(|Z|>t))^N \leq \exp (-N \mathbb{P}(|Z|>t))\;.
\end{equation}
We use the following Gaussian tail lower bound:
$$
\mathbb{P}(|Z|>t) \geq \frac{2}{\sqrt{2 \pi}} \frac{t}{1+t^2} e^{-t^2 / 2}
$$
Take $t=(1-\varepsilon) u_N$. When $N\to \infty$, we have $t \rightarrow \infty$ and $\frac{t}{1+t^2} \asymp \frac{1}{t}$. So for large $N$, 
$$
\mathbb{P}(|Z|>t) \geq c \frac{1}{u_N} \exp\paren*{-\frac{(1-\varepsilon)^2 u_N^2}{2}}=c \frac{1}{u_N} N^{-(1-\varepsilon)^2}\;.
$$
Therefore,
$$
N \mathbb{P}(|Z|>t) \geq c \frac{1}{u_N} N^{1-(1-\varepsilon)^2}=c \frac{1}{\sqrt{2\ln N}} N^{2 \varepsilon-\varepsilon^2} \underset{N \rightarrow \infty}{\longrightarrow} \infty,
$$
since $2\varepsilon - \varepsilon^2 > 0$. By \cref{eq:Max gaussian lower tail}, we have
$$
\mathbb{P}\left(\max _{i \leq N}\left|Z_i\right| \leq(1-\varepsilon) u_N\right) \underset{N \rightarrow \infty}{\longrightarrow} 0\;.
$$
By substituting the upper and lower tail bounds into \cref{eq:max-gauss-converge-in-prob}, we obtain the desired result.
\end{proof}

\subsection{Proof of \cref{lemma:EATS-is-UCB}}\label{sec:proof-of-EATS-is-UCB}

\begin{proposition}[Restatement of \cref{lemma:EATS-is-UCB}] For any $k$, when $N\to \infty$, it holds that
\begin{align}
    \max_{\phi \in \cA}\max_{i \in [N]}  
    \phi^\top \hat{\theta}_k + \frac{\beta_k}{\sqrt{ 2\ln N }} \lvert {\phi^\top \Lambda_k^{-1/2}\xi_k^i} \rvert
    \Pto \max_{\phi \in \cA} \phi^\top \hat{\theta}_k + \beta_k \norm*{\phi}_{\Lambda_k^{-1}}\;.
\end{align}
\end{proposition}
\begin{proof}
We condition on $\cF_{k-1}$. Without loss of generality, we prove the following claim: For fixed symmetric positive definite matrix $\Lambda \in \R^{d \times d}$, vector $\theta \in \mathbb{R}^d$ and $\beta > 0$, 
\begin{align}\label{eq:simple-prob-convergence}
    \max_{\phi \in \cA} \phi^\top \theta + \underbrace{\beta \max_{i \in [N]} \frac{1}{\sqrt{2\ln N}} \abs*{\phi^\top \Lambda^{-1/2}\xi^i}}_{\fd M_N(\phi)} \Pto \max_{\phi \in \cA} \phi^\top \theta + \beta \norm*{\phi}_{\Lambda^{-1}}\;.
\end{align}

\looseness=-1
Since $\cA$ is compact, pick a finite $\varepsilon$-net $\cN_\varepsilon$ of $\cA$ such that, for every $\phi \in \cA$, there is $\psi \in \cA_\varepsilon$ with $\norm*{\phi - \psi}_2 \leq \varepsilon$. From \cref{lemma:motivation-EATS} and union bound over $\psi \in \cA_\varepsilon$, we have
\begin{align}\label{eq:MN-union bound}
    \max_{\psi \in \cA_\varepsilon} \abs*{M_N(\psi) - \beta \norm*{\psi}_{\Lambda^{-1}}} \Pto 0\;.
\end{align}

\looseness=-1
We have the following Lipschitz continuity of $M_N$: for any $\phi_1, \phi_2 \in \cA$,
\begin{align*}
    \abs*{M_N(\phi_1) - M_N(\phi_2)} \leq \frac{\beta}{\sqrt{2\ln N}} \max_{i \in [N]} \abs*{(\phi_1 - \phi_2)^\top \Lambda^{-1/2}\xi^i} \leq \norm*{\phi_1 - \phi_2}_2 L_N\;,
\end{align*}
where we defined
$$
L_N \df \frac{\beta}{\sqrt{2\ln N}} \max_{i \in [N]} \norm*{\Lambda^{-1/2}\xi^i}_2 \leq
\beta \norm*{\Lambda^{-1/2}}_{\mathrm{op}} \frac{1}{\sqrt{2\ln N}} \max_{i \in [N]} \norm*{\xi^i}_2 \Pto \beta \norm*{\Lambda^{-1/2}}_{\mathrm{op}}\;.
$$
Here, $\|\cdot\|_{\mathrm{op}}$ denotes the operator norm and the convergence in the last step follows from \cref{lemma:motivation-EATS}. Additionally, $\phi \mapsto \|\phi\|_{\Lambda^{-1}}$ is $\|\Lambda^{-1/2}\|_{\mathrm{op}}$-Lipschitz continuous. Therefore, by combining the above Lipschitz continuity results with \cref{eq:MN-union bound}, we have
\begin{align*}
    \max_{\phi \in \cA} \abs*{M_N(\phi) - \beta \norm*{\phi}_{\Lambda^{-1}}} &\leq \max_{\psi \in \cA_\varepsilon} \abs*{M_N(\psi) - \beta \norm*{\psi}_{\Lambda^{-1}}} + O(\varepsilon) \Pto O(\varepsilon)\;.
\end{align*}

\looseness=-1
Finally, define $F_N(\phi) \df \phi^\top \theta + M_N(\phi)$ and $F(\phi) \df \phi^\top \theta + \beta \norm*{\phi}_{\Lambda^{-1}}$. We have
\begin{align*}
    \abs*{\max_{\phi \in \cA} F_N(\phi) - \max_{\phi \in \cA} F(\phi)} \leq 
   \max_{\phi \in \cA} \abs*{F_N(\phi) - F(\phi)} \leq \max_{\psi \in \cA_{\varepsilon}} \abs*{M_N(\psi) - \beta \norm*{\psi}_{\Lambda^{-1}}} + O(\varepsilon) \Pto O(\varepsilon)\;.
\end{align*}
By letting $\varepsilon \to 0$, we obtain the desired result.
\end{proof}

\newpage
\section{Additional Experimental results}
\begin{figure}[ht]
    \centering

    \begin{minipage}{0.48\linewidth}
        \centering
        \includegraphics[width=\linewidth]{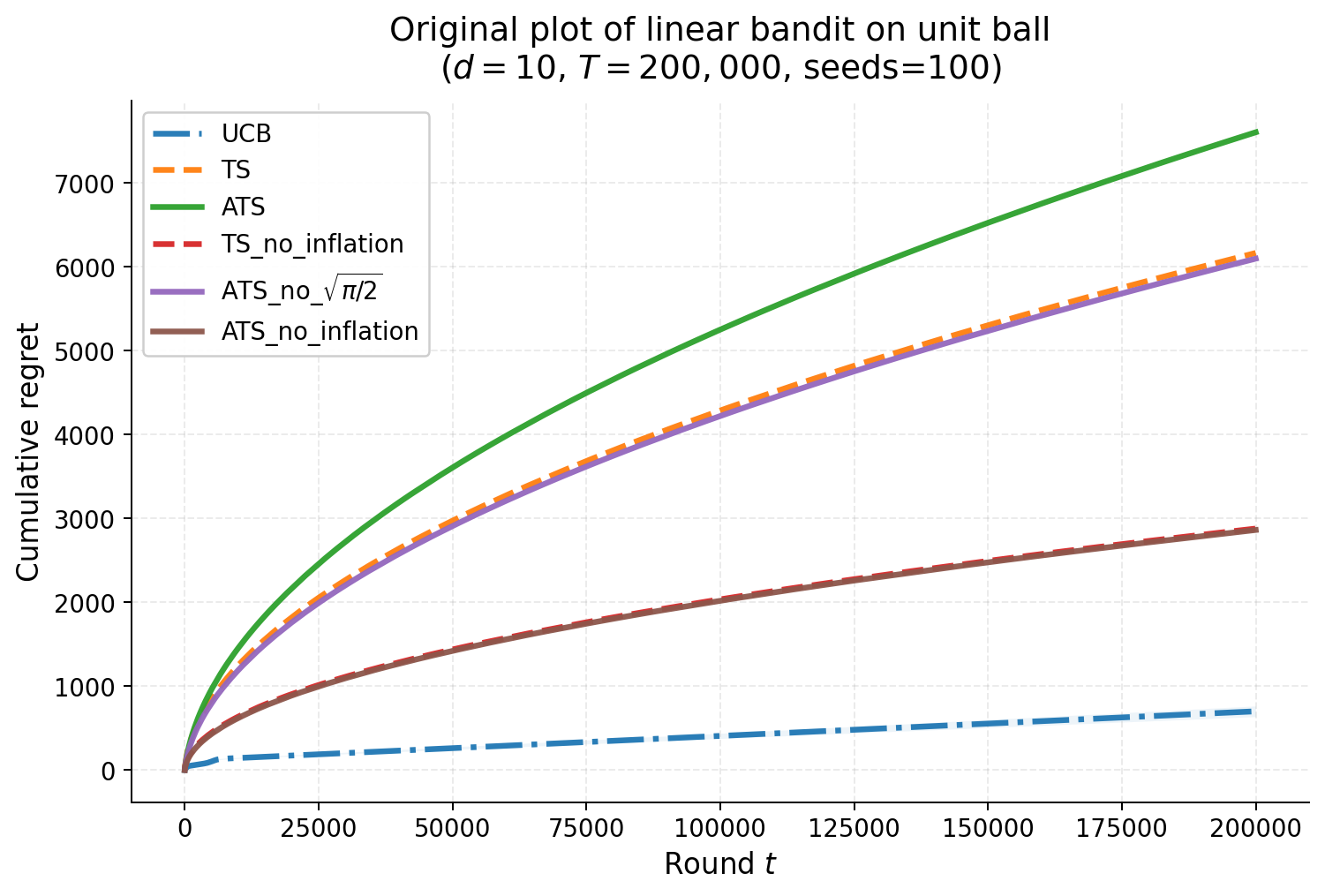}
    \end{minipage}
    \hfill
    \begin{minipage}{0.48\linewidth}
        \centering
        \includegraphics[width=\linewidth]{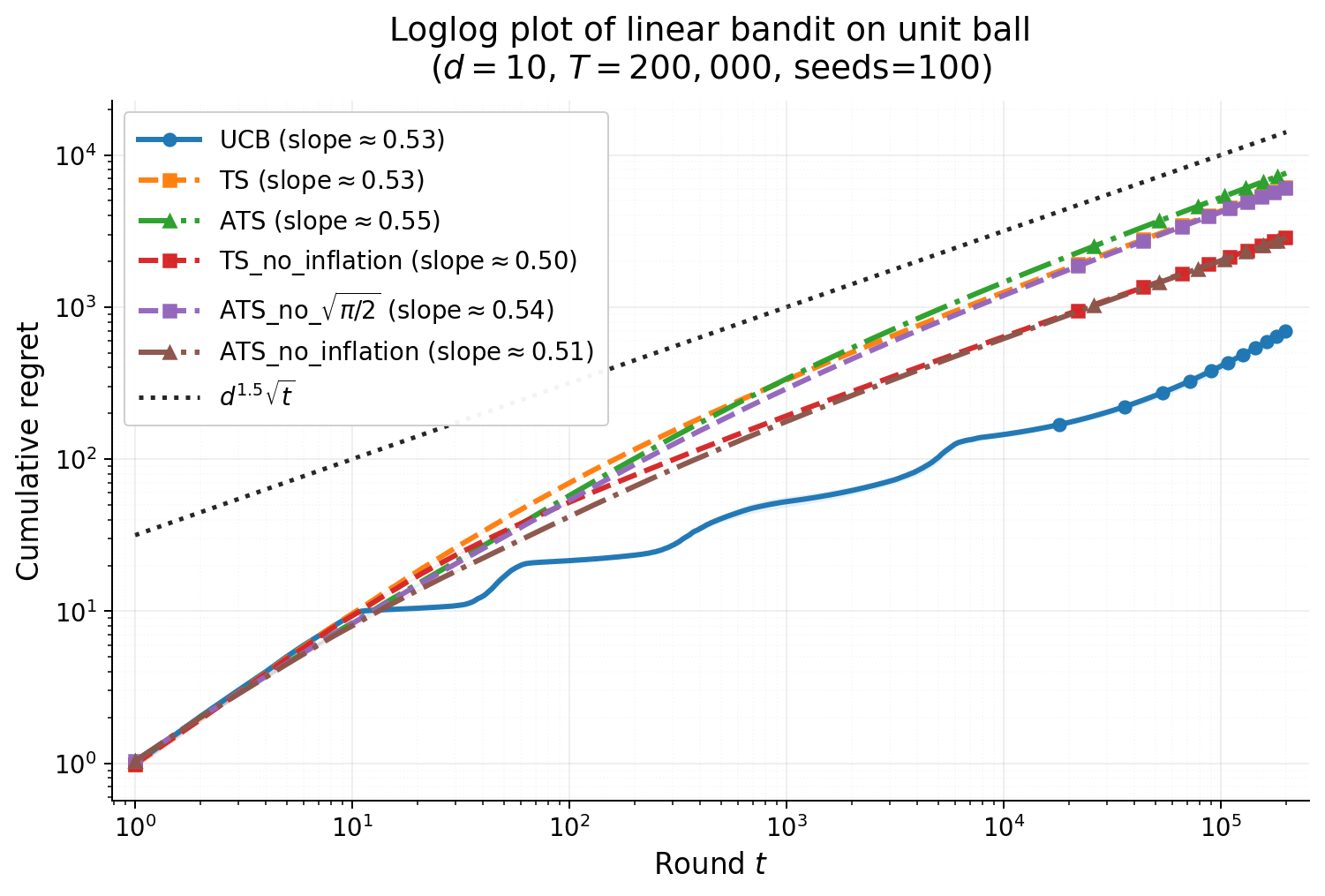}
    \end{minipage}

    \vspace{0.8em}


    \caption{\small
    Numerical results for linear bandits on the unit ball with feature dimension $d=10$, horizon $T=200000$, $\|\theta^\star\|=1$, and noise $\eta_t \sim \mathcal N(0,0.1)$. The parameter $\theta^\star$ is sampled uniformly from the unit sphere, and regret is averaged over $100$ independent seeds. 
    \textit{(Left)} Regret curves of UCB, TS, TS with no inflation, ATS, ATS with no $\sqrt{\pi/2}$ factor and  ATS and the right column shows log-log plots of cumulative regret.  
    In the right column, the black dotted line is $R_{\mathrm{ref}}(t)=d^{1.5}\sqrt{t}$ for reference.}
    \label{fig:unitsphere-three-rows}
\end{figure}
\begin{figure}[ht]
    \centering

    \begin{minipage}{0.48\linewidth}
        \centering
        \includegraphics[width=\linewidth]{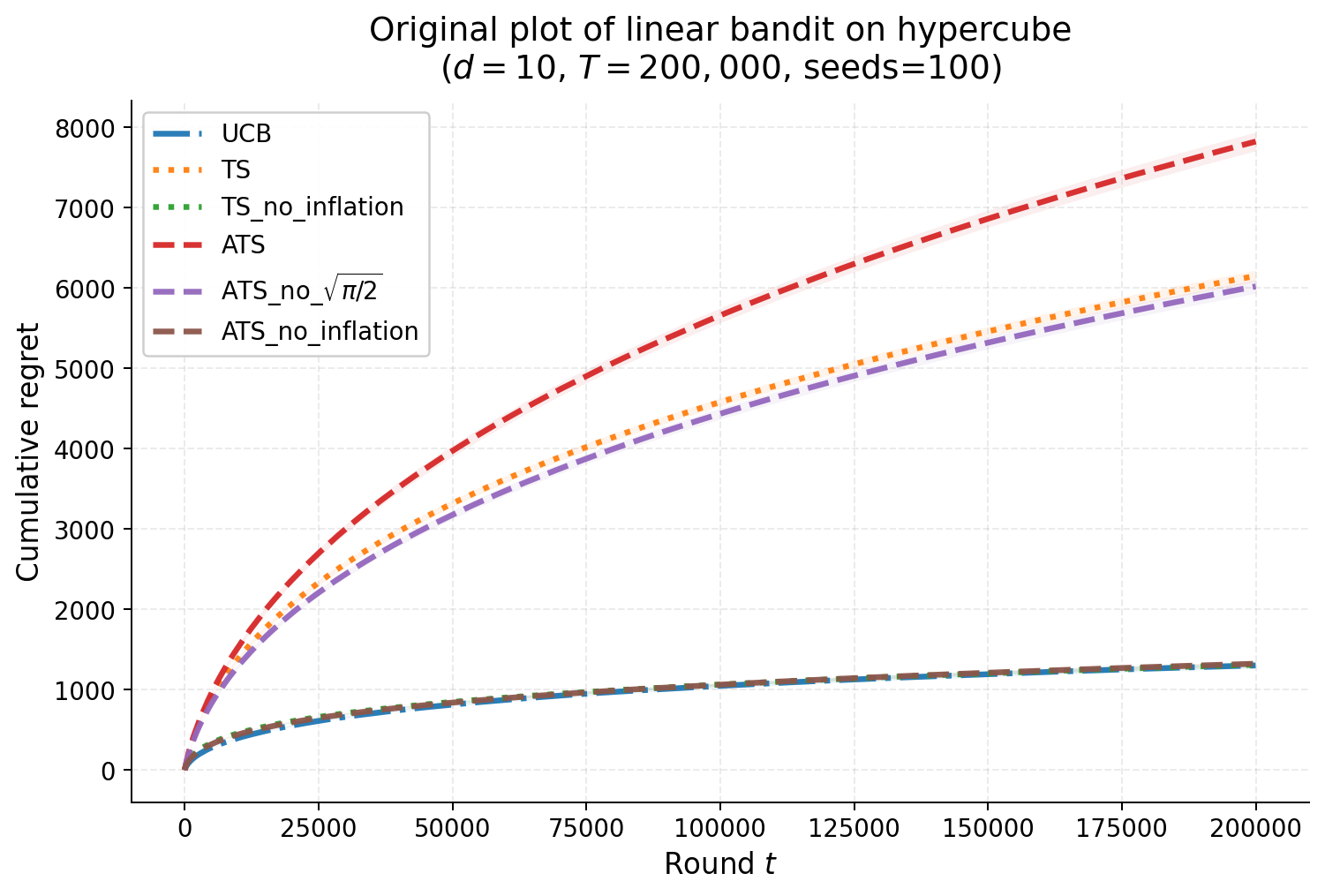}
    \end{minipage}
    \hfill
    \begin{minipage}{0.48\linewidth}
        \centering
        \includegraphics[width=\linewidth]{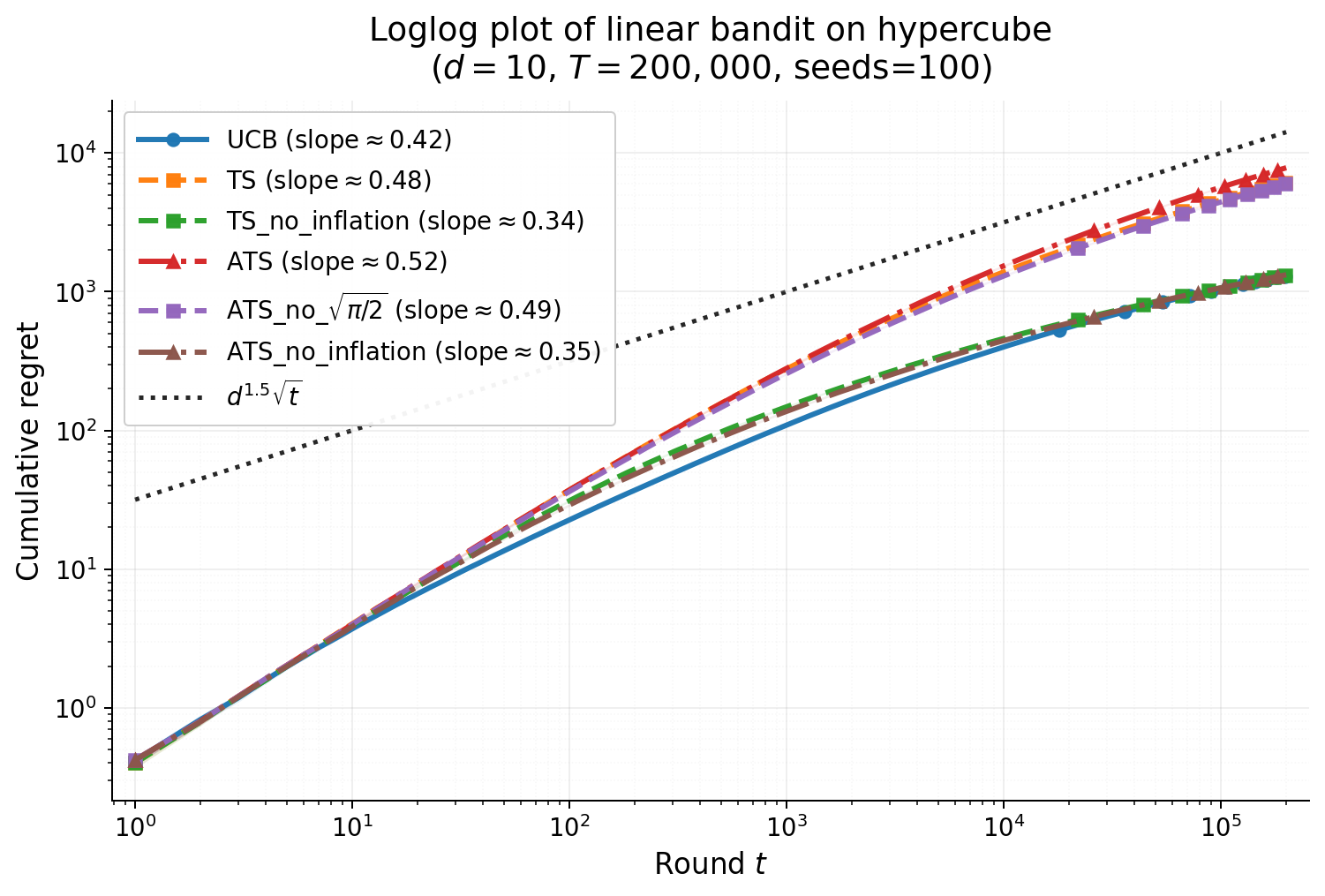}
    \end{minipage}


    \caption{\small
    Numerical results for linear bandits on the hypercube with feature dimension $d=10$, horizon $T=200000$, $\|\theta^\star\|=1$, and noise $\eta_t \sim \mathcal N(0,0.1)$. The parameter $\theta^\star$ is sampled uniformly from the unit sphere, and regret is averaged over $100$ independent seeds. 
    \textit{(Left)} Regret curves of UCB, TS, TS-no-inflation, ATS, ATS-no-inflation, ATS-no-$\sqrt{\pi/2}$. \textit{(Right)} Log-log plots of the regret curves. The black dotted line is $R_{\mathrm{ref}}(t)=d^{1.5}\sqrt{t}$ for reference.}
    \label{fig:hypercube-three-rows}
\end{figure}

\begin{figure}
    \centering
    \includegraphics[width=0.5\linewidth]{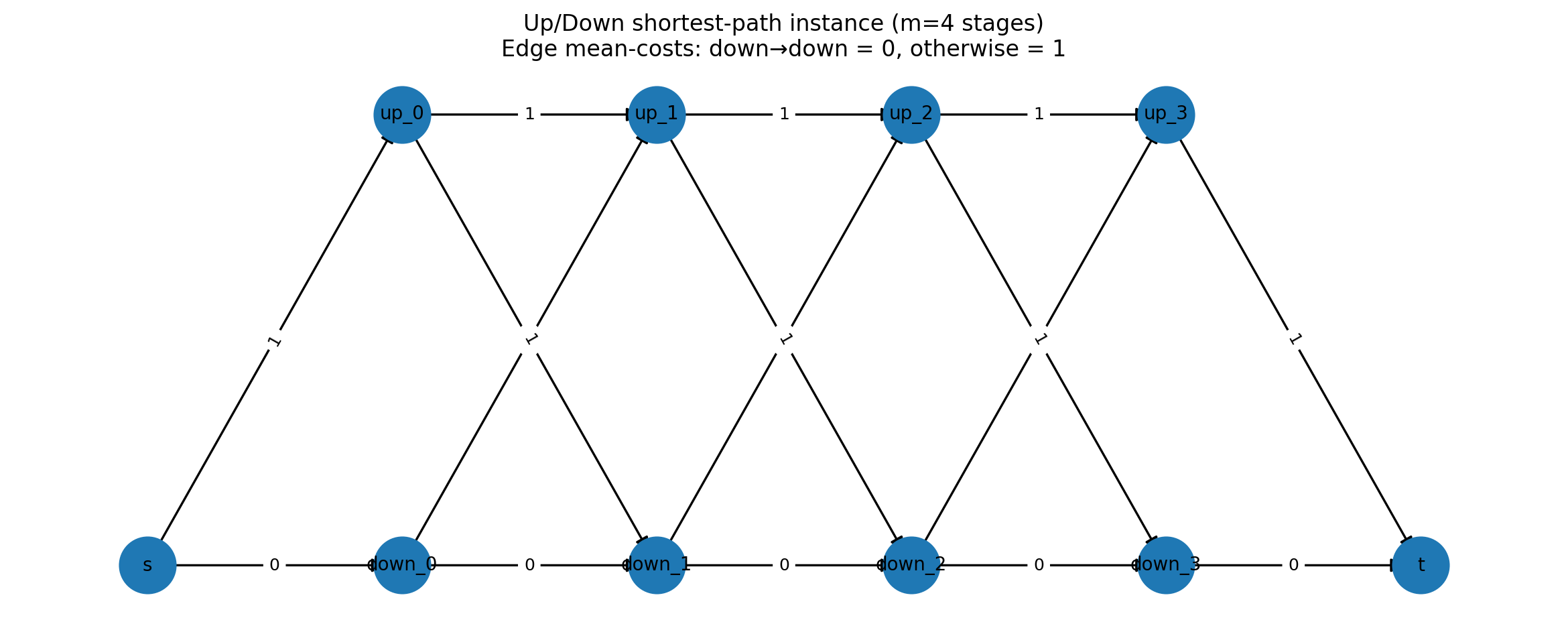}
    \caption{\small Example of the two-lane shortest path problem with $m=4$}
    \label{fig:shortest-path}
\end{figure}
\end{document}